\documentclass[12pt]{article}
\usepackage{amsmath,amsfonts,amsthm,mathrsfs}

\newtheorem{theorem}{Theorem}[section]
\newtheorem{lemma}[theorem]{Lemma}
\newtheorem{example}[theorem]{Example}

\newtheorem{proposition}[theorem]{Proposition}

\theoremstyle{definition}
\newtheorem{definition}[theorem]{Definition}

\numberwithin{equation}{section}

\def\cX{\mathcal X}
\def\cY{\mathcal Y}
\def\cZ{\mathcal Z}
\def\RR{\mathbb R}
\def\NN{\mathbb N}
\def\ZZ{\mathbb Z}
\def\EE{\mathbb E}
\def\bE{\mathbf E}
\def\bP{\mathbf  P}
\def\bz{\mathbf z}

\def\px{\rho_{\!_X}}
\def\L2p{{L^2_{\rho_{\!_X}}}}
\def\H{\mathcal H}
\def\calE{\mathcal E}
\def\calH{\mathcal H}
\def\calR{\mathscr R}
\def\calS{\mathcal S}
\def\calA{\mathcal A}
\def\calN{\mathcal N}
\def\calF{\mathcal F}
\def\calP{\mathcal P}
\def\ri{{\rm i}}
\def\e{{\mathbf e}}
\def\id{\mathbf 1}

\def\dint{\displaystyle\int}
\def\dsum{\displaystyle\sum}

\title{Consistency Analysis of an Empirical Minimum Error Entropy Algorithm\thanks{
~The work described in this paper is supported by National Natural Science
Foundation of China under Grants (No.  11201079, 11201348, and 11101403) and by a
grant from the Research Grants Council of Hong Kong [Project No.
CityU 104012]. Jun Fan (junfan@stat.wisc.edu) is with the
Department of Statistics,
University of Wisconsin-Madison, 1300 University Avenue
Madison, WI 53706 USA. 
Ting Hu
(tinghu@whu.edu.cn) is with School of Mathematics and Statistics,
Wuhan University, Wuhan 430072, China. Qiang Wu
(qwu@mtsu.edu) is with Department of Mathematical Sciences,
Middle Tennessee State University, Murfreesboro, TN
37132, USA.
 Ding-Xuan
Zhou (mazhou@cityu.edu.hk) are with Department of Mathematics, City
University of Hong Kong, Kowloon, Hong Kong, China. }}

\author{Jun Fan, Ting Hu, Qiang Wu\thanks{Corresponding author. 
Tel: +1 615 898 2053; Fax: +1 615 898 5422; Email: qwu@mtsu.edu} 
and Ding-Xuan Zhou}

\date{}

\begin{document}
\maketitle

\begin{abstract} 
In this paper we study the consistency of an empirical minimum error entropy (MEE) 
algorithm in a regression setting.
We introduce two types of consistency. The error entropy consistency,
which requires the error entropy of the learned function
to approximate the minimum error entropy, is shown to be always true
if the bandwidth parameter tends to 0 at an appropriate rate. The
regression consistency, which requires the learned function
to approximate the regression function, however, is a complicated
issue. We prove that the error entropy consistency implies the
regression consistency for homoskedastic models where the noise is
independent of the input variable. But for heteroskedastic models, a
counterexample is used to show that the two types of consistency do not
coincide. A surprising result is that the regression consistency is
always true, provided that the bandwidth parameter tends to infinity
at an appropriate rate. 
%This result, however, contradicts the motivation
%of the MEE principle because the minimum error entropy is believed to lead to bad approximations with this choice of bandwidth parameter.
Regression consistency of two classes of special models 
is shown to hold with fixed bandwidth parameter, which further illustrates the complexity
of regression consistency of MEE. Fourier transform plays crucial roles in our analysis.

\end{abstract}

\section{Introduction}

Information theoretical learning (ITL) is an important research area
in signal processing and machine learning. It uses  concepts of
entropies and divergences from information theory to substitute the conventional
statistical descriptors of variances and covariances. The idea dates
back at least to \cite{Link88} while its blossom was inspired by a
series of works of Principe and his collaborators. In \cite{EP00} the minimum
error entropy (MEE) principle was introduced to regression problems.
Later on its computational properties were studied and its
applications in feature extraction, clustering, and blind source
separation were developed \cite{EP2002, GP2002, EHP2002, EP2003}.
More recently the MEE principle was applied to classification
problems \cite{SMA2005, SMA2010}. For a comprehensive survey and
more recent advances on ITL and the MEE principle, see \cite{Prin:2010}
and references therein.

The main purpose of this paper is rigorous consistency analysis of
an empirical MEE algorithm for regression. Note that the ultimate goal 
of regression problems is the prediction on unobserved data or forecasting the future.
Consistency analysis in terms of predictive powers  is deemed to be important 
to interpret  the effectiveness of a regression algorithm. 
The empirical MEE has been developed and successfully applied in various fields for more than a
decade and there are some theoretical studies in the literature which
provide good understanding of computational complexity of the empirical
MEE and its parameter choice strategy. However, 
the consistency of the MEE algorithm, especially from a prediction perspective, is lacking. 
In our earlier work \cite{HFWZ2012}, we proved the consistency of the MEE algorithm
in a special situation, where we require the algorithm utilizes a large bandwidth parameter.
The motivation of the MEE algorithm (to be describe below) is to minimize the error entropy which requires a 
small bandwidth parameter. The result in \cite{HFWZ2012} is somewhat contradictory to this motivation.
An interesting question is whether the MEE algorithm is consistent in terms of predictive powers 
if a small bandwidth parameter is chosen as implied by its motivation. Unfortunately,
this is not a simple `yes' or `no' question. Instead, the consistency of the MEE algorithm is a very complicated issue.
In this paper we will try to depict 
a full picture on it -- establishing the relationship between the
error entropy and a $L_2$ metric measuring the predictive powers, and providing conditions such that 
the MEE algorithm is predictively consistent.

%But they are not from asymptotic
%perspectives and cannot explain the effectiveness of the empirical
%MEE algorithm as the sample size gets large. In this paper we will
%analyze the algorithm from a statistical learning theory
%perspective. Our asymptotic analysis will help establish the
%consistency of the empirical MEE in several different situations. It
%turns out the consistency of the empirical MEE is a very complicated
%issue which explains its difficulty. Some analysis has been conducted in our
%earlier work \cite{HFWZ2012} for a special case (of large bandwidth parameter,
%mentioned later in more details).

%We will focus on a regression setting in learning theory. 
In statistics a regression problem is usually modeled as the
estimation of a target function $f^*$ from a metric space
$\cX$ to the another metric space $\cY\subset \RR$ for which a set of
observations $(x_i, y_i),\ i=1,\ldots, n$, are obtained from a model
\begin{equation} Y = f^*(X) + \epsilon, \qquad \bE (\epsilon|X)=0. \label{model} \end{equation}
In the statistical learning context \cite{Vapn:1998}, the regression
setting is usually described as the learning of the regression function
which is defined as conditional mean $\bE(Y|X)$ of the output
variable $Y$ for given input variable $X$ under the assumption that
there is an unknown joint probability measure $\rho$ on the product
space $\cX\times \cY.$  These two settings are equivalent by
noticing that
$$ f^*(x) = \bE(Y|X=x).$$
A learning algorithm for regression produces a function $f_\bz$ from
the observations $\bz=\{(x_i,y_i)\}_{i=1}^{n}$ as an approximation of $f^*$.
The goodness of this approximation can be measured by certain
distance between $f_\bz$ and $f^*$, for instance,
$\|f_\bz-f^*\|_{\L2p}$, the $L^2$ distance with respect to the
marginal distribution $\px$ of $\rho$ on $\cX$.

MEE algorithms for regression are motivated by minimizing some
entropies of the error random variable $E=E(f)=Y-f(X)$,
where $f:X\to \RR$ is a hypothesis function. In this paper we
focus on the R\'enyi's entropy of order $2$ defined as
\begin{eqnarray}\label{Renyi}
\calR(f) = -\log\left(\bE [p_{_E}]\right)= -\log\left(\int_{\RR} (p_{_E}(e))^2\,de\right).
\end{eqnarray}
Here and in the sequel, $p_{_E}$ is the probability density function
of $E.$ since $\rho$ is unknown, we need an empirical estimate of $p_E$. 
Denote $e_i=y_i-f(x_i)$. Then $p_{_E}$ can be estimated from
the sample $\bz$ by a kernel density estimator by using a Gaussian kernel
$G_h(t)=\frac 1{\sqrt{2\pi} h } \exp\left(-\frac{t^2}{2h^2}\right)$ with
bandwidth parameter $h$:
$$ p_{\!_{E,\bz}}(e) = \frac 1{n}\sum_{j=1}^n G_h(e-e_j) =
\frac 1{n} \sum_{j=1}^n  \frac 1{\sqrt{2\pi} h }
\exp\left(-\frac{(e-e_j)^2}{2h^2}\right).$$
The MEE algorithm produces an appropriate $f_\bz$ from a set $\H$ of continuous functions on $\cX$ called the hypothesis space by minimizing the
empirical version of the R\'enyi's entropy
$$ \calR_\bz(f) = -\log \left(\frac 1n \sum_{i=1}^n p_{\!_{E,\bz}} (e_i) \right)
= -\log\left(\frac 1{n^2}\sum_{i=1}^n \sum_{j=1}^n G_h(e_i-e_j) \right).$$ That
is, $f_\bz=\arg\min\limits_{f\in\H}\calR_\bz(f)$. It is obvious that
minimizers of $\calR$ and $\calR_\bz$ are not unique because
$\calR(f)=\calR(f+b)$ and $\calR_\bz(f)=\calR_\bz(f+b)$ for any
constant $b$. Taking this into account, $f_\bz$ should be adjusted
by a constant when it is used as an approximation of the
regression function $f^*.$

The empirical entropy $\calR_\bz(f)$ involves an empirical mean 
$\frac{1}{n}\sum_{i=1}^n p_{\!_{E,\bz}} (e_i) $ which makes it look like an M-estimator.
However, the density estimator $p_{E, \bz}$ itself is data dependent, 
making the MEE algorithm different from standard M-estimations, 
with two summation indices involved.
This can be seen from our earlier work \cite{HFWZ2012} 
where we used U-statistics for the error analysis in the case of large parameter $h.$

To study the asymptotic behavior of the MEE algorithm we define two types of consistency as follows:
\begin{definition} The MEE algorithm is {\bf{consistent with respect to the R\'enyi's error entropy}}
if $\calR(f_\bz)$ converges to $\calR^* =
\inf\limits_{f:\cX\rightarrow\RR} \calR(f)$ in probability as
$n\to\infty$, i.e.,
$$\lim_{n\to\infty}\bP\Big(\calR (f_\bz)-\calR^*>\varepsilon\Big)=0,
\qquad \forall\varepsilon > 0.$$
The MEE algorithm is {\bf{consistent with respect to the regression
function}} if $f_\bz$ plus a suitable constant adjustment  converges
to $f^*$ in probability with the convergence measured in the $\L2p$
sense, i.e., there is a constant $b_\bz$ such that $f_\bz+b_\bz$
converges to $f^*$ in probability, i.e.,
$$\lim_{n\to\infty}\bP\Big(\|f_\bz+b_\bz-f^*\|_\L2p^2>\varepsilon\Big)=0,\qquad \forall\varepsilon > 0.$$
\end{definition}

Note that the error entropy consistency ensures the learnability of the
minimum error entropy, as is expected from the motivation of
empirical MEE algorithms. However, the error entropy itself is 
not a metric that directly measures the predictive powers of the algorithm. 
(We assume that  a metric $d$ measuring the predictive powers should be monotone
in the sense that $|E(f_1)|\le |E(f_2)|$ implies $d(f_1)\le d(f_2).$ Error entropy 
is clearly not such a metric.) 

To measure the predictive consistency, one may choose different metrics. In the definition of 
regression function consistency we have adopted the $L_2$ distance to the 
true regression function $f^*$, the target function of the regression problem. 
The regression consistency guarantees good approximations of the regression target function $f^*$
and thus serves as a good measure for predictive powers.
%These two types of consistency, however, are not necessarily coincident. % Instead, they may contradict each other.

Our main results,  stated in several theorems in Section 2 below, involve two main contributions. 
(i) We characterize the relationship between the error entropy consistency and regression consistency. 
We prove that the error entropy consistency implies the regression function consistency only
for very special cases, for instance, the homoskedastic models,
while in general this is not true. For heteroskedastic models, 
 a counterexample is used to show that the error entropy consistency
and regression consistency is not necessary to coincide.
(ii) We prove a variety of consistency results for the MEE algorithm. 
 Firstly we prove that the error entropy
consistency is always true by choosing the bandwidth parameter $h$
to tend to $0$ slowly enough. As a result, 
the regression function consistency holds for the homoskedastic models.
Secondly, for heteroskedastic models, regression consistency is shown to be incorrect if the 
bandwidth parameter is chosen to be small. 
But we restate the result from \cite{HFWZ2012}
which shows that the empirical MEE
is always consistent with respect to the regression function
if the bandwidth parameter is allowed to be chosen large enough. 
%This
%was observed in some earlier empirical work but clearly contradicts
%the motivation of MEE algorithms because Parzen windowing for the
%minimum error entropy does not lead to convergence without $h\to 0.$ 
Lastly, we consider two classes of special regression models for which the regression consistency can be true with fixed choices of the bandwidth parameter $h$. These results indicate that the consistency of the empirical MEE is a
very complicated issue. % and can only be well understood after extensive investigations.

\section{Main results}

We state our main results in this section while giving their proofs
later. We need to make some assumptions for analysis
purposes. Two main assumptions, on the regression model and the
hypothesis class respectively, will be used throughout the
paper.  

For the regression model, we assume some natural regularity conditions. 
\begin{definition}
The regression model \eqref{model} is {\bf{MEE admissible}}  if
\begin{enumerate}
\item the density function $p_{\epsilon|X}$ of the noise variable $\epsilon$ for given $X=x\in\cX$ exists and is uniformly bounded by a constant $M_p$;
\item the regression function $f^*$ is bounded by a constant $M>0;$
\item  the minimum of $\calR(f)$ is achieved by a measurable function $f^*_\calR$.
\end{enumerate}
\end{definition}

Note that we do not require the boundedness or exponential decay of the noise term as in the usual setting of learning theory. It is in fact an advantage of MEE to allow heavy tailed noises. 
Also, it is easy to see that if $f_\calR^*$ is a minimizer, then for any constant $b$,
$f_\calR^*+b$ is also a minimizer.  So we cannot assume the
uniqueness of $f^*_\calR$. Also, no obvious relationship
between $f^*$ and $f^*_\calR$ exists. Figuring out such a non-trivial relationship is
one of our tasks below. We also remark that some results below may hold under relaxed conditions, 
but for simplifying our statements, we will not discuss them in detail.

Our second assumption is on the hypothesis space which is
required to be a learnable class and have good approximation ability with respect to the target function.
\begin{definition}\label{hypospace} We say the hypothesis space $\calH$ is {\bf{MEE admissible}} if
\begin{enumerate}
\item $\calH$ is uniformly bounded, i.e., there is a constant $M$ such that $|f(x)|\le M$ for all $f\in\calH$ and
all $x\in\cX$;
\item the $\ell_{2}$-norm empirical cover number $\calN_2(\calH, \varepsilon)$ (see Appendix or \cite{Bart:Mend:2002, wu2009} for its definition) satisfies $\log\calN_2(\calH, \varepsilon)\le c\varepsilon^{-s}$
for some constant $c>0$ and some index $0<s<2$;
\item a minimizer $f_\calR^*$ of $\calR(f)$ and the regression function $f^*$ are in
$\calH$.
\end{enumerate}
\end{definition}
The first condition in Definition \ref{hypospace} is common in the literature and is natural since we do not expect to
learn unbounded functions. The second condition ensures $\calH$
is a learnable class so that overfitting will not happen. This is often imposed in learning theory. It is also easily
fulfilled by many commonly used function classes. The third condition
guarantees the target function can be well approximated by $\calH$
for otherwise no algorithm is able to learn the target function well
from $\calH$. Although this condition can be relaxed to that the
target function can be approximated by function sequences in
$\calH$, we will not adopt this relaxed situation for simplicity.

Throughout the paper, we assume that both the regression model \eqref{model} and 
the hypothesis space $\calH$ are MEE admissible. Our first main result is to verify the error entropy consistency.
\begin{theorem}\label{thm:entropyconsistency}
If the bandwidth parameter $h=h(n)$ is chosen to satisfy
\begin{equation}\label{echcond}
\lim_{n\to\infty} h(n)= 0, \qquad \lim_{n\to\infty} h^2\sqrt{n} =
+\infty,
\end{equation}
 then
$\calR (f_\bz)$ converges to $\calR^*$ in probability.

If, in addition, the derivative of $p_{\epsilon|X}$ exists and is uniformly bounded by a constant $M'$
independent of $X$, then by choosing $h(n)\sim n^{-\frac 16}$, for any $0<\delta<1$, with probability at least $1-\delta$, we have
\begin{eqnarray*}
\calR(f_\bz)-\calR^*=O(\sqrt{\log(2/\delta)}n^{-\frac16}).
\end{eqnarray*}
\end{theorem}

In the literature of practical implementations of MEE, the optimal choice of $h$ is
suggested to be $h(n)\sim n^{-\frac 15}$ (see e.g.
\cite{Prin:2010}). We see this choice satisfies our condition for
the error entropy consistency. But the optimal rate analysis is out
of the scope of this paper.

The error entropy consistency in Theorem
\ref{thm:entropyconsistency} states the minimum error entropy can be
approximated with a suitable choice of the bandwidth parameter. This
is a somewhat expected result because empirical MEE algorithms are
motivated by minimizing the sample version of the error entropy risk
functional. However, later we will show that this does not
necessarily imply the consistency with respect to the regression
function. Instead, the regression consistency is a complicated
problem. Let us discuss it in two different situations.

%Firstly, we show that for a very special class of noise models the error entropy consistency does imply the regression consistency.
%\begin{definition} The regression model has symmetric unimodal noise if  $p_{\epsilon|x}$ is even and decreasing when $\epsilon>0.$
%\end{definition}

\begin{definition}
The regression mode \eqref{model} is {\bf{homoskedastic}} if the noise
$\epsilon$ is independent of $X$. Otherwise it is said to be
{\bf{heteroskedastic}}.
\end{definition}
Our second main result states the regression consistency for homoskedastic models.
\begin{theorem}\label{thm:homo}
If the regression model is homoskedastic, then the following holds.
\begin{enumerate}
\item  $\calR^*=\calR(f^*)$. As a result, for any constant $b$, $f_\calR^*=f^*+b$ is a minimizer of $\calR(f)$;
\item there is a constant $C$ depending on $\rho, \calH$ and M such that, for any measurable function $f$,
$$\|f+\bE(f^*-f)-f^*\|_\L2p^2 \le C\left(\calR(f)-\calR^*\right);$$
\item if \eqref{echcond} is true, then
$f_\bz + \bE_x(f^*-f_\bz)$ converges to $f^*$ in probability;
\item if, in addition, the derivative of $p_{\epsilon|X}$ exists and is uniformly bounded by a constant $M'$ independent of $X$, then the convergence rate of order $O(\sqrt{\log(2/\delta)}n^{-\frac 16})$
can be obtained with confidence $1-\delta$ for $\|f_\bz+\bE_x(f^*-f_\bz)-f^*\|_\L2p^2$ by choosing $h\sim n^{-\frac 16}$.
\end{enumerate}
\end{theorem}

Theorem \ref{thm:homo} (iii) shows the regression consistency for
homoskedastic models. It is a corollary of
error entropy consistency stated in Theorem \ref{thm:entropyconsistency}
and the relationship between the $\L2p$ distance and the excess
error entropy stated in Theorem \ref{thm:homo} (ii). Thus the homoskedastic
model is a special case for which the error entropy consistency and
regression consistency coincide with each other.

Things are much more complicated for heteroskedastic models. Our
third main result illustrates the incoincidence of the minimizer
$f_\calR^*$  and the regression function $f^*$ by Example \ref{exam:incoin} in section 5.
\begin{proposition}\label{prop:incoincidence}
There exists a heteroskedastic model such that the regression
function $f^*$ is not a minimizer of $\calR(f)$ and the regression
consistency fails even if the error entropy consistency is true.
\end{proposition}
This result shows that, in general,  the error entropy consistency
does not  imply the regression consistency. Therefore, these two
types of consistency do not coincide for heteroskedastic models.

However, this observation does not mean the empirical MEE algorithm cannot be consistent with respect to
the regression function. In fact, in \cite{HFWZ2012} we proved the regression consistency for large bandwidth parameter $h$ and derived learning rate when $h$ is of the form $h=n^\theta$ for some $\theta>0$.

Our fourth main result in this paper is to verify the regression consistency for a more general choice of large bandwidth parameter $h$.
\begin{theorem}\label{thm:regressionconsistency}
Choosing the bandwidth parameter $h=h(n)$ such that
\begin{equation}\label{rchcond}
\lim_{n\to\infty} h(n)= +\infty, \qquad \lim_{n\to\infty} \frac {h^2}{\sqrt{n}} = 0,
\end{equation}
we have
$f_\bz + \bE_x(f^*(x)-f_\bz(x))$  converges to $f^*$ in probability.
A convergence rate of order $O(\sqrt{\log(2/\delta)}n^{-\frac 14})$ can be obtained with confidence $1-\delta$ for $\|f_\bz+\bE_x(f^*-f_\bz)-f^*\|_\L2p^2$ by taking $h\sim n^{\frac 18}.$
\end{theorem}

Such a result looks surprising. Note that the empirical MEE algorithm
is motivated by minimizing an empirical version of the error
entropy. This empirical error entropy approximates the true one when
$h$ tends to zero. But the regression consistency is in general true
as $h$ tends to infinity, a condition under which the error entropy
consistency may not be true. From this point of view, the regression
consistency of the empirical MEE algorithm does not justify its
motivation.

Observe that the regression consistency in Theorem \ref{thm:homo}
and Theorem \ref{thm:regressionconsistency} suggests the constant adjustment to be
$b=\bE_x[f^*(x)-f_\bz(x)]$. In practice
the constant adjustment is usually taken as
$\dfrac{1}{n} \sum\limits_{i=1}^n (y_i-f_\bz(x_i))$
which is exactly the sample mean of $b.$ 

The last two main results of this paper are about the regression consistency of two special classes of regression models. 
We show that the bandwidth parameter $h$ can be chosen as a fixed positive constant to make MEE consistent in these situations. 
Moreover the convergence rate is of order $O(n^{-1/2})$, much higher than previous general cases. Throughout this paper, 
we use $\ri$ to denote the imaginary unit and
$\overline {a}$ the conjugate of a complex number $a.$
The Fourier transform $\widehat{f}$ is defined for an integrable function f on $\RR$ as
$\widehat{f}(\xi)=\int_{\RR}f(x)\e^{-\ri x\xi}dx.$   Recall the inverse Fourier transform is given by $f(x)=\frac1{2\pi}\int_{\RR}\widehat{f}(\xi)\e^{\ri x\xi}d\xi$ when $f$ is square integrable. Fourier transform plays crucial roles in our analysis.

\begin{definition}
\label{unimodal}
A univariate function $f$ is {\bf{unimodal}} if for some $t\in\RR$, the function is monotonically increasing on $(-\infty, t]$ and monotonically decreasing on $[t, \infty)$.
\end{definition}

\begin{definition}
We define $\calP_1$ to be the set of probability measures $\rho$ on $\cX\times\cY$ satisfying the following conditions:
\begin{enumerate}
\item $p_{\epsilon|X=x}$ is symmetric (i.e. even) and unimodal for every $x\in\cX$;
\item the Fourier transform $\widehat{p_{\epsilon|X=x}}$ is nonnegative on $\RR$ for every $x\in\cX$;
\item there exist two constants $c_0>0$ and $C_0>0$ such that $\widehat{p_{\epsilon|X=x}}(\xi)\geq C_0$ for $\xi\in[-c_0, c_0]$ and every $x\in\cX$.
\end{enumerate}
We define $\calP_2$ to be the set of probability measures $\rho$ on $\cX\times\cY$ such that
$p_{\epsilon|X=x}$ is symmetric for every $x\in\cX$ and there exists some constant $\widetilde{M}>0$ such that $p_{\epsilon|X=x}$ is supported on $[-\widetilde{M}, \widetilde{M}]$ for every $x\in\cX.$
\end{definition}

The boundedness assumption on the noise the for the family $\calP_2$ is very natural in regression setting.
For the family $\calP_1$, the conditions look complicated, but the following two examples tell that 
they are also common in statistical modeling.

\begin{example}
(Symmetric $\alpha$-stable L\'evy distributions) A distribution is said to be symmetric $\alpha$-stable L\'evy distributions \cite{nolan:2012} 
if it is symmetric and its Fourier transform is represented in the form $e^{-\gamma^\alpha|\xi|^\alpha}$, with $\gamma>0$ and $0<\alpha\leq2$. Obviously, Gaussian distribution with mean zero is a special case with $\alpha=2$. Cauchy distribution with median zero is another special case with $\alpha=1$. Every distribution in this set is unimodal \cite{Laha:1961}. If we choose a subset of these distributions with $\gamma\leq C$ (C is a constant), then the Fourier transform is positive and $\exists c_0=1/C$ and $C_0=e^{-1}$ such that $\forall \xi\in[-c_0, c_0]$, $\widehat{p_{\epsilon|X}}(\xi)\geq C_0$.
\end{example}
\begin{example}
(Linnik distributions) A Linnik distribution is also referred to as a symmetric geometric stable distribution \cite{Koz:Pod:Sam:1999}. A distribution is said to be Linnik distribution if it is symmetric and its Fourier transform is represented in the form $\frac{1}{1+\lambda^\alpha|\xi|^\alpha}$,with $\lambda>0$ and $0<\alpha\leq2$. Obviously, Laplace distribution with mean zero is a special case with $\alpha=2$. Every distribution in this set is unimodal \cite{Laha:1961}. If we choose a subset of these distributions with $\lambda\leq C$ (C is a constant), then the Fourier transform is positive and $\exists c_0=1/C$ and $C_0=\frac12$ such that $\forall \xi\in[-c_0, c_0]$, $\widehat{p_{\epsilon|X}}(\xi)\geq C_0$.
\end{example}

Corresponding to the definition of the empirical R\'enyi's entropy $\calR_\bz(f)$, after removing the logarithm, we define information error of a measurable function $f: \cX\rightarrow\RR$ as
$$\begin{array}{rl}
\calE_h(f)  & = -\dint_\RR\dint_\RR G_h(e-e')p_{\!_E}(e)p_{\!_E}(e')dede' \\
& = -
\dint_\cZ\dint_\cZ G_h\Big((y-f(x))-(y'-f(x'))\Big) d\rho(x,y)d\rho(x',y').
\end{array}$$

\begin{theorem}
\label{thm:symuninoise}
If $\rho$ belongs to $\calP_1$, then $f^*+b$ is a minimizer of $\calE_h(f)$ for any constant b and any fixed $h>0$. Moreover, we have $f_\bz+\bE_x(f^*-f_\bz)$ converges to $f^*$ in probability. Convergence rate of order $O(\sqrt{\log(2/\delta)}n^{-\frac12})$ can be obtained with confidence $1-\delta$ for $\|f_\bz+\bE_x(f^*-f_\bz)-f^*\|_\L2p^2$.
\end{theorem}

\begin{theorem}
\label{thm:boundnoise}
If $\rho$ belongs to $\calP_2$, then there exists some $h_{\rho, \calH}>0$ such that $f^*+b$ is a minimizer of $\calE_h(f)$ for any fixed $h> h_{\rho,\calH}$ and constant b. Also $f_\bz+\bE_x(f^*-f_\bz)$ converges to $f^*$ in probability. Convergence rate of order $O(\sqrt{\log(2/\delta)}n^{-\frac12})$ can be obtained with confidence $1-\delta$ for $\|f_\bz+\bE_x(f^*-f_\bz)-f^*\|_\L2p^2$.
\end{theorem}

\section{Error entropy consistency}

In this section we will prove that $\calR(f_\bz)$ converges to $\calR^*$ in probability when $h=h(n)$ tends to zero slowly satisfying \eqref{echcond}. Several useful lemmas are needed to prove our first main result (Theorem \ref{thm:entropyconsistency}).

% Let $\rho_{_X}$ denote the marginal distribution of $\rho$ on $\cX$.
\begin{lemma}\label{lem:pe}
For any measurable function $f$ on $\cX$, the probability density function
for the error variable $E=Y-f(X)$ is given as
\begin{equation}\label{pe}
p_{\!_E}(e) = \int_\cX p_{\epsilon|X}(e+f(x)-f^*(x)|x)d\rho_{_X}(x).
\end{equation}
As a result, we have $|p_{\!_E}(e)|\le M$ for every $e\in\RR$.
\end{lemma}
\begin{proof} The equation \eqref{pe} follows from the fact that
$$\epsilon=Y-f^*(X)=E+f(X)-f^*(X).$$ The inequality $|p_{\!_E}(e)|\le M$
 follows from the assumption $|p_{\epsilon|X}(t)|\le M$.
\end{proof}

Denote by $B_L$ and $B_U$ the lower bound and upper bound of $\bE[p_{\!_E}]$ over $\calH$, i.e.,
$$ B_L=\inf_{f\in\calH} \int_\RR (p_{\!_E}(e))^2de\quad \hbox{ and }\quad
 B_U=\sup_ {f\in\calH} \int_\RR (p_{\!_E}(e))^2de.$$
\begin{lemma}\label{lem:BLU}
We have $0<B_L$ and $B_U\le M_p.$
\end{lemma}
\begin{proof} Since $\dint_\cX\dint_{-\infty}^{\infty}p_{\epsilon|X}(t|x)dt d\rho_X(x)=1$, there is some constant $0<A<+\infty$ such that
 $$a=\int_\cX\int_{-A}^A p_{\epsilon|X}(t|x)dt d\rho_{_X}(x)>\frac 12.$$
For any $f\in\calH$, by the fact $|f|\le M$ and $|f^*|\le M$, it is easy to check form \eqref{pe} that
$$\begin{array}{rl}
\dint_{-(A+2M)}^{A+2M} p_{\!_E}(e) de &
= \dint_\cX \dint_{-(A+2M)}^{A+2M} p_{\epsilon|X}(e+f(x)-f^*(x)|x) de d\rho_{_X}(x) \\
& = \dint_\cX \dint_{-(A+2M)+f(x)-f^*(x)}^{A+2M+f(x)-f^*(x)} p_{\epsilon|X}(t|x) dt d\rho_{_X}(x) \\
& \ge \dint_{\cX}\dint_{-A}^A p_{\epsilon|X}(t|x) dt d\rho_{_X}(x)
=a.
\end{array}$$
Then by the Schwartz inequality we have
$$\begin{array}{rl} a\le & \dint_{-(A+2M)}^{A+2M} p_{\!_E}(e) de \le \left(\dint_{-(A+2M)}^{A+2M} (p_{\!_E}(e))^2de\right)^{\frac 12}\left(\dint_{-(A+2M)}^{A+2M} de\right)^{\frac 12} \\
& \le \sqrt{2A+4M}\left(\dint_\RR (p_{\!_E}(e))^2de\right)^{\frac 12} .\end{array}$$
This gives
$$\dint_\RR (p_{\!_E}(e))^2de \ge \frac {a^2}{2A+4M} \ge \frac{1}{8A+16M}$$ for any
$f\in\calH$. Hence $B_L\ge \frac{1}{8A+16M}>0.$

The second inequality follows from
the fact that $p_{\!_E}$ is a density function and uniformly bounded by $M_p$. This proves Lemma \ref{lem:BLU}.
\end{proof}

It helps our analysis to remove the logarithm from the R\'enyi's entropy \eqref{Renyi} and define \begin{eqnarray}\label{VRenyi}
V(f) = -\bE[p_{\!_E}] = -\int (f_{\!_E}(e))^2\,de.
\end{eqnarray}
Then
$\calR(f)=-\log(-V(f)).$ Since $-\log(-t)$ is strictly increasing
for $t\leq0,$ minimizing $\calR(f)$ is equivalent to minimizing
$V(f).$ As a result, their minimizers are the same. Denote
$V^*=\inf\limits_{f:\cX\rightarrow\RR}V(f)$. Then $V^*(f) = -\log(-\calR^*),$ and we
have the following lemma.
\begin{lemma} \label{RVcompare}
For any $f\in\calH$ we have
$$\frac 1{B_U}\Big(V(f)-V^*\Big) \le \calR(f)-\calR^*\le \frac 1{B_L}
\Big(V(f)-V^*\Big).$$
\end{lemma}

\begin{proof} Since the derivative of the function $-\log(-t)$ is $-\frac 1t$, by the mean value theorem we get
$$\calR(f)-\calR^*=\calR(f)-\calR(f_\calR^*) = -\log(-V(f))-[-\log(-V(f_\calR^*))] = -\frac 1{\xi} \left(V(f)-V(f_\calR^*)\right)$$
for some $\xi\in[V(f_\calR^*), V(f)] \subset[-B_U, -B_L].$ This leads to the conclusion.
\end{proof}

From Lemma \ref{RVcompare} we see that, to prove Theorem \ref{thm:entropyconsistency},
it is equivalent to prove the convergence of $V(f_\bz)$ to $V^*.$ To this end we define an empirical version of the generalization error $\calE_{h,\bz}(f)$ as
$$\calE_{h,\bz}(f) = -\frac 1{n^2}\sum_{i,j=1}^n G_h(e_i-e_j)
 = -\frac 1{n^2}\sum_{i,j=1}^n G_h\Big((y_i-f(x_i)-(y_j-f(x_j))\Big).$$
Again we see the equivalence between minimizing $\calR_\bz(f)$ and
minimizing $\calE_{h,\bz}(f).$ So $f_\bz$ is also a minimizer of
$\calE_{h,\bz}$ over the hypothesis class $\calH.$ We then can bound
$V(f_\bz)-V^*$ by an error decomposition as
$$\begin{array}{rl}
V(f_\bz)-V^* & = \Big(V(f_\bz)-\calE_{h,\bz}(f_\bz)\Big) +
\Big(\calE_{h,\bz}(f_\bz)-\calE_{h,\bz}(f_\calR^*)\Big) +
\Big(\calE_{h,\bz}(f_\calR^*)-V(f_\calR^*)\Big) \\
&\le 2\sup\limits_{f\in\calH}\left|\calE_{h,\bz}(f)-V(f)\right|
\le 2 \calS_\bz + 2\calA_h.
\end{array}$$
where $\calS_\bz$ is called the sample error defined by $\calS_\bz=\sup\limits_{f\in\calH}\left|\calE_{h,\bz}(f)-\calE_h(f)\right|$ and $\calA_h$ is called approximation error defined by $\sup\limits_{f\in\calH}\left|\calE_{h}(f)-V(f)\right|.$

The sample error $\calS_\bz$ depends on the sample, and can be estimated by the following proposition.
\begin{proposition} \label{prop:szbound} There is a constant $B>0$ depending on M, c and s (in Definition \ref{hypospace}) such that for every $\epsilon_1>0$,
$$ \bP\left(\calS_\bz>\varepsilon_1+\dfrac{B}{h^2\sqrt n}\right) \le
\exp(-2nh^2\varepsilon_1^2).$$
\end{proposition}
This proposition implies that $\calS_\bz$ is bounded by
$O\left(\frac 1{h^2\sqrt{n}}+\frac 1{h\sqrt{n}}\right)$ with large
probability. The proof of this proposition is long and complicated.
But it is rather standard in the context of learning theory. So we
leave it in the appendix where the constant B will be given explicitly.

The approximation error is small when h tends to zero, as shown in next proposition.
\begin{proposition} \label{prop:Ahbound}
We have $\lim\limits_{h\to 0} \calA_h=0.$
If the derivative of $p_{\epsilon|X}$ is uniformly bounded by a constant $M'$, then $\calA_h\leq M'h$.
\end{proposition}

\begin{proof}
Since $G_h(t)=\frac1h G_1(\frac{t}{h})$, by changing the variable $e'$ to $\tau=\frac{e-e'}{h}$, we have
\begin{eqnarray*}
\calA_h &= & \sup\limits_{f\in\calH}
\left|\dint_\RR\dint_\RR \frac1h G_1(\frac{e-e'}{h})p_{\!_E}(e)p_{\!_E}(e') de de' - \dint_\RR (p_{\!_E}(e))^2de\right|\\
 & = &   \sup\limits_{f\in\calH}
\left|\dint_\RR\dint_\RR
G_1(\tau)p_{\!_E}(e-\tau h)d\tau p_{\!_E}(e)
de - \dint_\RR (p_{\!_E}(e))^2de \right|
\end{eqnarray*}
But $\int_\RR G_1(\tau)d\tau=1$, we see from \eqref{pe} that
\begin{eqnarray}
\calA_h & = &   \sup\limits_{f\in\calH}
\left|\dint_\RR  p_{\!_E}(e)  \dint_\RR
G_1(\tau)(p_{\!_E}(e-\tau
h)-p_{\!_E}(e))d\tau de\right|
\nonumber \\
&\le &  \sup\limits_{f\in\calH} \dint_\RR  p_{\!_E}(e)  \dint_\RR
G_1(\tau)\dint_\cX \Big|p_{\epsilon|X}(e-\tau h+f(x)-f^*(x)|x)-
\nonumber\\
& & \qquad
-p_{\epsilon|X} (e+f(x)-f^*(x)|x)\Big| d\rho_{_X}(x) d\tau de.
\label{DVbound}
\end{eqnarray}
It follows form Lebesgue's Dominated Convergence Theorem that $\lim\limits_{h\to 0}\calA_h=0.$

If $|p'_{\epsilon|X}|\le M'$ uniformly for an $M'$, we have $$\left| p_{\epsilon|X}(e-\tau
h+f(x)-f^*(x)|x)-p_{\epsilon|X}
(e+f(x)-f^*(x)|x)\right| \le M'|\tau| h.$$ Then from \eqref{DVbound}, we find
$$\calA_h\leq\sup\limits_{f\in\calH}\dint_\RR p_{\!_E}(e)de\dint_\RR G_1(\tau)|\tau|d\tau M'h
=\frac{2M'}{\sqrt{2\pi}}h\leq M'h.$$ This proves Proposition \ref{prop:Ahbound}.
\end{proof}

We are in a position to prove our first main result Theorem \ref{thm:entropyconsistency}.

\noindent\begin{proof}[Proof of Theorem \ref{thm:entropyconsistency}]
Let $0<\delta<1$. By take $\varepsilon_1>0$
such that $\exp(-2nh^2\varepsilon_1^2)=\delta$, i.e., $\varepsilon_1=\sqrt{\frac{\log(1/\delta)}{2nh^2}}$, we know from Proposition \ref{prop:szbound} that with probability at least $1-\delta$,
$$\calS_\bz\leq \varepsilon_1+\frac{B}{h^2\sqrt n}=\frac{1}{h^2\sqrt n}(B+\sqrt{\log(1/\delta)}h).$$
To prove the first statement, we apply assumption \eqref{echcond}. For any $\varepsilon>0$, there exists some $N_1\in\NN$ such that $(B+1)\frac1{h^2\sqrt n}<\frac{\varepsilon}{2}$ and $\sqrt{\log(1/\delta)}h\leq1$ whenever $n\geq N_1$. It follows that with probability at least $1-\delta$, $\calS_\bz<\frac{\varepsilon}{2}$.
By proposition \ref{prop:Ahbound} and $\lim\limits_{n\rightarrow\infty}h(n)=0$, there exists some $N_2\in\NN$ such that $\calA_h\leq\frac{\varepsilon}{2}$ whenever $n\geq N_2$. Combining the above two parts for $n\geq\max\{N_1, N_2\}$, we have with probability at least $1-\delta$,
$$V(f_\bz)-V^*\leq2\calS_\bz+2\calA_h\leq 2\varepsilon,$$
which implies by Lemma \ref{RVcompare},
$$\calR(f_\bz)-R^*\leq\frac{2}{B_L}\varepsilon.$$
Hence the probability of the event $\calR(f_\bz)-R^*\geq\frac{2}{B_L}\varepsilon$ is at most $\delta$. This proves the first statement of Theorem \ref{thm:entropyconsistency}.

To prove the second statement, we apply the second part of Proposition \ref{prop:Ahbound}. Then with probability at least $1-\delta$, we have
$$R(f_\bz)-\calR^*\leq\frac{1}{B_L}(V(f_\bz)-V^*)\leq\frac{2}{B_L}\left(\frac1{h^2\sqrt n}(B+\sqrt{\log(1/\delta)}h)+M'h\right).$$
Thus, if $C'_1n^{-\frac16}\leq h(n)\leq C'_2n^{-\frac16}$ for some constants $0<C'_1\leq C'_2$, we have with probability at least $1-\delta$,
$$\calR(f_\bz)-\calR^*\leq\frac{1}{B_L}(V(f_\bz)-V^*)\leq\frac{2}{B_L}\left(\frac1{(C'_1)^2}(B+C'_2\sqrt{\log(1/\delta)})
+M'C'_2\right)n^{-\frac16}.$$
Then the desired convergence rate is obtained. The proof of Theorem \ref{thm:entropyconsistency} is complete.
\end{proof}

\section{Regression consistency for homoskedastic models}

In this section we prove the regression consistency for
homoskedastic models stated in Theorem \ref{thm:homo}. Under the
homoskedasticity assumption, the noise $\epsilon$ is independent of
$x$, so throughout this section we will simply use $p_\epsilon$ to
denote the density function for the noise. Also, we use the
notations $E=E(f) = Y-f(X)$ and $E^*=Y-f^*(X).$

The probability density function of the random
variable $E=Y - f(X)$ is given by
$$ p_{\!_E} (e) = \int_{{\mathcal X}}
p_\epsilon(e + f(x) -f^* (x)) d \px (x). $$ Then
$$ \int_{\RR} (p_{\! E} (e))^2 d e = \int_{{\mathcal X}} \int_{{\mathcal X}}
\int_{\RR} p_\epsilon(e + f(x) -f^* (x)) p_\epsilon(e + f(u) -f^* (u)) d e d
\px (x) d \px (u).
$$
We apply the Planchel formula and find
$$\int_{\RR} p_\epsilon(e + f(x) -f^* (x)) p_\epsilon(e + f(u) -f^* (u)) d e =
\frac{1}{2 \pi} \int_{\RR} \widehat{p_\epsilon}(\xi) \e^{\ri \xi (f(x) -f^*
(x))} \overline{\widehat{p_\epsilon} (\xi) \e^{\ri \xi (f(u) -f^* (u))}} d
\xi. $$ It follows that
$$ \int_{\RR} (p_{\!E} (e))^2 d e = \frac{1}{2 \pi}\int_{{\mathcal X}} \int_{{\mathcal X}}
\int_{\RR} |\widehat{p_\epsilon}(\xi)|^2 \e^{\ri \xi (f(x) -f^* (x) - f(u)
+f^* (u))}d \xi d \px (x) d \px (u).
$$
This is obviously maximized when $f=f^*$ since $|\e^{\ri\xi t}| \leq
1$. This proves that $f^*$ is a minimizer of $V(f)$ and $\calR(f).$
Since $V(f)$ and $\calR(f)$ are invariant with respect to constant
translates, we have proved part (i) of Theorem \ref{thm:homo}.

To prove part (ii), we study the excess quantity
$V(f) - V(f^*)$ and express it as
$$\begin{array}{rl} V(f) - V(f^*) & = \dint_{\RR}
(p_{\!_{E^*}} (e))^2 d e - \dint_{\RR} (p_{\! _E} (e))^2 d e
\\
& = \dfrac{1}{2 \pi}\dint_{{\mathcal X}} \dint_{{\mathcal X}}
\dint_{\RR} |\widehat{p_\epsilon}(\xi)|^2 \left(1- \e^{\ri \xi (f(x) -f^*
(x) - f(u) +f^* (u))}\right)d \xi d \px (x) d \px (u) \\
&  = \dfrac{1}{2 \pi}\dint_{{\mathcal X}} \dint_{{\mathcal X}}
\int_{\RR} |\widehat{p_\epsilon}(\xi)|^2 2 \sin^2 \dfrac{\xi (f(x) -f^*
(x) - f(u) +f^* (u))}{2}d \xi d \px (x) d \px (u)
\end{array}$$
where the last equality follows from the fact that $V(f)-V(f^*)$ is real
and hence equals to its real part.

As both $f$ and $f^*$ take values on $[-M, M]$, we know that
$|f(x) -f^* (x) - f(u) +f^* (u)| \leq 4 M$ for any $x, u \in
{\mathcal X}$. So when $|\xi| \leq \frac{\pi}{4M}$, we have
$$\left|\frac{\xi (f(x) -f^*
(x) - f(u) +f^* (u))}{2}\right| \leq \frac{\pi}{2}$$ and
$$\left|\sin \frac{\xi (f(x) -f^*
(x) - f(u) +f^* (u))}{2}\right| \geq \frac{2}{\pi}\left|\frac{\xi
(f(x) -f^* (x) - f(u) +f^* (u))}{2}\right|. $$
Observe that the integrand in the expression of $V(f) - V(f^*)$ is
nonnegative and
\begin{eqnarray*} && \int_{\RR} |\widehat{p_\epsilon}(\xi)|^2 2 \sin^2 \frac{\xi (f(x) -f^*
(x) - f(u) +f^* (u))}{2}d \xi \\
&\geq& \int_{|\xi| \leq \frac{\pi}{4M}} |\widehat{p_\epsilon}(\xi)|^2 2
\sin^2 \frac{\xi (f(x) -f^* (x) - f(u) +f^* (u))}{2}d \xi
\\
&\geq& \int_{|\xi| \leq \frac{\pi}{4M}} |\widehat{p_\epsilon}(\xi)|^2
\frac{2}{\pi^2} \xi^2 \left(f(x) -f^* (x) - f(u) +f^* (u)\right)^2
d \xi.
\end{eqnarray*}
Therefore, \begin{eqnarray*} V(f) - V(f^*) &\geq& \frac{1}{\pi^3}
\int_{|\xi| \leq \frac{\pi}{4M}}\xi^2 |\widehat{p_\epsilon}(\xi)|^2 d \xi
\int_{{\mathcal X}} \int_{{\mathcal X}} \left(f(x) -f^* (x) - f(u)
+f^* (u)\right)^2  d \px (x) d \px (u).
\end{eqnarray*}
It was shown in \cite{HFWZ2012}  that
\begin{equation} \label{varf} \int_{{\mathcal X}} \int_{{\mathcal X}} \left(f(x) -f^* (x) - f(u)
+f^* (u)\right)^2  d \rho_X (x) d \rho_X (u) = 2
\|f-f^*+\bE(f^*-f)\|_{\L2p}^2.
\end{equation}  So we have
$$V(f) - V(f^*) \geq \left( \frac{2 }{\pi^3}
\int_{|\xi| \leq \frac{\pi}{4M}}\xi^2 |\widehat{p_\epsilon}(\xi)|^2 d \xi\right)
\|f-f^*+\bE(f^*-f)\|_{\L2p}^2. $$

Since the probability density function $p_\epsilon$ is integrable, its
Fourier transform $\widehat{p_\epsilon}$ is continuous. This together
with $\widehat{p_\epsilon}(0)=1$ ensures that $\widehat{p_\epsilon}(\xi)$
is nonzero over a small interval around $0.$
As a result $\xi^2 |\widehat{p_\epsilon}(\xi)|^2$
is not identically zero on $[-\frac {\pi}{4M},\, \frac {\pi}{4M}]$.
Hence the constant $$c=
\int_{|\xi| \leq \frac{\pi}{4M}}\xi^2 |\widehat{p_\epsilon}(\xi)|^2 d \xi$$ is positive
and the conclusion in (ii) is proved by taking $C=\frac {\pi^3B_U}{2c}$ and applying Lemma \ref{RVcompare}.

Parts (iii) and (iv) are easy corollaries of part (ii) and Theorem \ref{thm:entropyconsistency}.
This finishes the proof of Theorem \ref{thm:homo}.

%tells us that there exists some $b^* \in
%(0, \frac{\pi}{4M}]$ such that for any $\xi \in [-b^*, b^*]$, we
%have $|\widehat{p_\epsilon}(\xi) - \widehat{p_\epsilon} (0)| \leq \frac{1}{2}$
%which implies $\left|\widehat{p_\epsilon}(\xi)\right| \geq \frac{1}{2}$.
%Hence
%$$\int_{|\xi| \leq \frac{\pi}{4M}}\xi^2 |\widehat{p_\epsilon}(\xi)|^2 d
%\xi \geq \int_{|\xi| \leq b^*}\xi^2 |\widehat{p_\epsilon}(\xi)|^2 d \xi
%\geq \int_{|\xi| \leq b^*}\xi^2 \frac{1}{4} d \xi =
%\frac{(b^*)^3}{6}.
%$$ It follows that
%$$V(f) - V(f^*) \geq \frac{(b^*)^3}{3 \pi^3}
%\||f-f_\rho|\|_{L^2_{\rho_X}}^2. $$ This proves the statement with
%$c^* = \frac{(b^*)^3}{3 \pi^3}$.

\section{Incoincidence between error entropy consistency and regression consistency}

In the previous section we proved that for homoskedastic models the
error entropy consistency implies the regression consistency. But
for heteroskedastic models, this is not necessarily true. Here we
present a counter-example to show this incoincidence between two
types of consistency.

Let $\id_{(\cdot)}$ denote the indicator function on a set specified by the subscript.

\begin{example}\label{exam:incoin} Let ${\mathcal X} =\cX_1\bigcup\cX_2=[0, \frac 12] \bigcup [1, \frac 32]$ and
$\px$ be uniform on $\cX$ (so that $d\px=dx$). The conditional distribution of $\epsilon|X$
is uniform on $[-\frac 12, \frac 12]$ if  $x\in [0, \frac 12]$ and
 uniform on $[-\frac 32, -\frac 12]\bigcup [\frac 12,
\frac 32]$ if $x\in [1, \frac 32]$. Then
\begin{enumerate}
\item a function $f^*_\calR: {\mathcal X} \to
\RR$ is a minimizer of $\calR(f)$ if and only if there are two constants $f_1,\ f_2$ with $|f_1-f_2|=1$
such that $f_\calR^*=f_1\mathbf 1_{\cX_1} +f_2\mathbf 1_{\cX_2};$
\item $\calR^*=-\log(\frac 5 8)$ and $\calR(f^*)=-\log(\frac 38).$ So the regression function $f^*$
is not a minimizer of the error entropy functional $\calR(f);$
\item let $\calF_\calR^*$ denote the set of all minimizers. There is an a constant $C'$ depending on $\calH$ and M such that for any measurable function $f$ bounded by $M$,
$$\min_{f_\calR^*\in\calF_\calR^*} \|f-f_\calR^*\|_\L2p^2 \le C'\Big(\calR(f)-\calR^*\Big);$$
\item if the error entropy consistency is true, then there holds
$$\min_{f_\calR^*\in\calF_\calR^*} \|f_\bz-f_\calR^*\|_\L2p\longrightarrow 0 \qquad \hbox{and} \qquad
\min_{b\in\RR} \|f_\bz+b-f^*\|_\L2p \longrightarrow \frac 12$$
in probability. As a result, the regression consistency cannot be true.
\end{enumerate}
\end{example}

\begin{proof} Without loss of generality we may assume $M\ge 1$ in this example.

Denote $p_1(\epsilon)=p_{\epsilon|X}(\epsilon|x)$
for $x\in\cX_1$ and $p_2(\epsilon)=p_{\epsilon|X}(\epsilon|x)$ for $x\in\cX_2.$
By Lemma \ref{lem:pe}, the probability density function of  $E=Y - f(X)$ is given by
$$ \begin{array}{rl} p_{\!_E} (e) = & \dint_{{\mathcal X}}
p_{\epsilon|X}(e + f(x) -f^* (x)|x) d \px (x)
= \dsum_{j=1}^2\dint_{{\mathcal X}_j}
p_{j}(e + f(x) -f^* (x)) d x  . \end{array}$$
 So we have
$$ \begin{array}{rl} \dint_{\RR} (p_{\!_E} (e))^2 d e =
\dsum_{j, k=1}^2 \int_{{\mathcal X}_j} \int_{{\mathcal X}_k}
\int_{\RR} p_{j} (e + f(x) -f^* (x)) p_{k} (e + f(u) -f^* (u)) d e d \px(x) d \px(u).
\end{array}
$$
%$$
%Denote the Fourier transform of a probability density function $p$ as $\widehat{p}$. We apply the Planchel formula and find
%$$\int_{\RR} p_{j} (t + f(x) -f^* (x)) p_{k} (t + f(u) -f^* (u)) d t =
%\frac{1}{2 \pi} \int_{\RR} \widehat{p_j}(\xi) e^{i \xi (f(x) -f^* (x))} \overline{\widehat{p_k} (\xi) e^{i \xi (f(u) -f^*
%(u))}} d \xi. $$ It follows that
%$$ \int_{\RR} (p_{Y - f(X)} (t))^2 d t
By the Planchel formula,
$$ \begin{array}{rl} \dint_{\RR} (p_{\!_E} (e))^2 d e = \dfrac{1}{2 \pi} \dsum_{j, k=1}^2 \int_{{\mathcal X}_j} \int_{{\mathcal X}_k}
\int_{\RR} \widehat{p_j}(\xi) \overline{\widehat{p_k} (\xi)}
\e^{\ri \xi (f(x) -f^* (x) - f(u) +f^* (u))}d \xi d \px (x) d
\px (u).
\end{array}
$$

Let $p^*=\id_{[-\frac 12, \frac 12]}$ be the density function of the uniform distribution on  $[-\frac 12, \frac 12] $. Then  we have  $p_1 = p^*$ and $p_2 (e) = \frac{p^* (e+1) + p^* (e-1)}{2}$ which yields
$$\widehat{p_2} (\xi) = \frac{\e^{-\ri \xi} + \e^{\ri \xi}}{2}
\widehat{p^*}(\xi) =  \widehat{p^*}(\xi) \cos \xi.$$
These together with $f^* \equiv 0$ allow us to write
\begin{equation}\label{decom}
V(f) = -\int_{\RR} (p_{\!_E} (e))^2 d e =V_{11}(f) + V_{22}(f)
+ V_{12}(f), \end{equation} where
\begin{eqnarray*} V_{11}(f) &=& -\frac{1}{2 \pi} \int_{{\mathcal X}_1} \int_{{\mathcal X}_1} \int_{\RR}
\left|\widehat{p^*}(\xi)\right|^2 \e^{\ri \xi (f(x) -
f(u))}d \xi d \rho_X (x) d \rho_X (u),\\
V_{22}(f) &=& -\frac{1}{2 \pi} \int_{{\mathcal X}_2}
\int_{{\mathcal X}_2} \int_{\RR} \cos^2 \xi
 \left|\widehat{p^*}(\xi)\right|^2 \e^{\ri \xi (f(x) -
f(u))}d \xi d \rho_X (x) d \rho_X (u),\\
V_{12}(f) &=& -\frac{1}{\pi} \int_{{\mathcal X}_1} \int_{{\mathcal
X}_2} \int_{\RR} \left|\widehat{p^*}(\xi)\right|^2 \cos \xi \cos
\left(\xi (f(x) - f(u))\right) d \xi d \rho_X (x) d \rho_X (u).
\end{eqnarray*}

Recall the following identity from Fourier analysis (see e.g. \cite{JM1991})
\begin{equation}\label{identity}
\sum_{\ell \in \ZZ}\widehat{p^*}(\xi + 2 \ell \pi)
\overline{\widehat{p^* (\cdot - b)}(\xi + 2 \ell \pi)} =\sum_{\ell
\in\ZZ} \langle p^*(\cdot - \ell), p^* (\cdot - b)\rangle_{L^2
(\RR)} \e^{\ri \ell \xi}, \quad \forall \xi, b\in \RR.
\end{equation}
In particular, with $b=0$, since the integer translates of $p^*$
are orthogonal, there hold $\sum_{\ell \in \ZZ}
\left|\widehat{p^*}(\xi + 2 \ell \pi)\right|^2 \equiv 1$ and
$$\int_{\RR}
\left|\widehat{p^*}(\xi)\right|^2 \cos^j \xi d \xi = \int_{[-\pi,
\pi)} \sum_{\ell \in \ZZ} \left|\widehat{p^*}(\xi + 2 \ell
\pi)\right|^2 \cos^j \xi d \xi =\begin{cases} 0, &
\quad\hbox{if} \ j=1, \\
 2 \pi, & \quad\hbox{if} \ j=0, \\
 \pi, &\quad\hbox{if} \ j=2. \end{cases}
$$

For $V_{11}(f)$, notice the real analyticity of the
function $\widehat{p^*}(\xi) = \frac{2 \sin (\xi/2)}{\xi}$ and the
identity $$\int_{\RR} \left|\widehat{p^*}(\xi)\right|^2 \e^{\ri \xi
(f(x) - f(u))}d \xi = \int_{\RR} \left|\widehat{p^*}(\xi)\right|^2
\cos \left(\xi (f(x) - f(u))\right)d \xi.$$ We see that $V_{11}(f)$ is
minimized if and only if $f(x)=f(u)$ for any $x,u\in \cX_1$. In this case, $f$ is a constant on ${\mathcal X}_1$, denoted as $f_1$, and the minimum value of $V_{11}(f)$ equals $$V_{11}^* :=-(\px
({\mathcal X}_1))^2 =-\frac{1}{4}.$$ Moreover, if a measurable
function satisfies $f(x) \in [-M, M]$ for every $x\in {\mathcal
X}_1$, we have
\begin{eqnarray} V_{11} (f) - V_{11}^* &=& \frac{1}{2 \pi}\int_{{\mathcal X}_1} \int_{{\mathcal X}_1}
\int_{\RR} |\widehat{p^*}(\xi)|^2 \Big(1- \cos \left(\xi (f(x) -
f(u))\right)\Big) d \xi d \px (x) d \px (u) \nonumber\\
&=& \frac{1}{2 \pi}\int_{\cX_1} \int_{\cX_1}
\int_{\RR} |\widehat{p^*}(\xi)|^2 2 \sin^2 \left(\dfrac{\xi (f(x) -
f(u)}{2}\right) d \xi d \px (x) d \px (u) \nonumber \\
&\geq& \frac{1}{2 \pi}\int_{{\mathcal X}_1} \int_{{\mathcal X}_1}
\int_{|\xi| \leq \frac{\pi}{4M}}
|\widehat{p^*}(\xi)|^2 2 \left(\frac{2}{\pi} \frac{\xi (f(x) -
f(u))}{2}\right)^2 d \xi d \px (x) d \px (u) \nonumber \\
&\geq& \frac{1}{24 \pi^2 M^3} \int_{{\mathcal X}_1}
\int_{{\mathcal X}_1} \left(f(x) - f(u)\right)^2 d \px(x) d
\px (u) \nonumber\\
& = & \frac{1}{12 \pi^2 M^3}
\left\|f - m_{f, {\mathcal X}_1}\right\|^2_{\L2p({\mathcal
X}_1)} \label{partIest}
\end{eqnarray}
where $$m_{f,\cX_j} = \frac{\bE[f\id_{\cX_j}]}{\px(\cX_j)} =
\frac{1}{\px(\cX_j)}\int_{\cX_j} f(x)d\px(x)$$ denotes the mean of
$f$ on $\cX_j.$

Similarly, $V_{22}(f)$ is minimized if and only if $f$ is constant on ${\mathcal X}_2$, which will be denoted as $f_2,$ and the corresponding minimum value equals $$V_{22}^* := -\frac{1}{2}(\px ({\mathcal X}_2))^2 =-\frac{1}{8}.$$ Again, if a measurable function satisfies
$f(x) \in [-M, M]$ for every $x\in {\mathcal X}_2$, we have \begin{equation}\label{partIIest} V_{22} (f) - V_{22}^* \geq
\frac{1}{24 \pi^2 M^3} \left\|f - m_{f, {\mathcal
X}_2}\right\|^2_{\L2p({\mathcal X}_2)}.
\end{equation}

For $V_{12}(f)$, we express it as
$$ V_{12}(f) = -\frac{1}{4 \pi} \int_{{\mathcal X}_1}\!\int_{{\mathcal
X}_2}\! \int_{\RR} \left|\widehat{p^*}(\xi)\right|^2
\left(\e^{\ri\xi} + \e^{-\ri\xi}\right) \! \left(\e^{\ri\xi (f (x) -
f (u))}\! + \e^{-\ri\xi (f (x) - f (u))}\right) d \xi d \px\! (x) d
\px\! (u). $$ Write $f (x) - f (u)$ as $k_{f, x, u} + b_{f, x, u}$
with $k_{f, x, u} \in \ZZ$ being the integer part of the real number of $f(x)-f(u)$ and $b_{f, x, u} \in [0, 1)$. We have
\begin{eqnarray*} &&\int_{\RR} \left|\widehat{p^*}(\xi)\right|^2 \left(\e^{\ri\xi} + \e^{-\ri\xi}\right)
\e^{\ri\xi (f(x) - f (u))}d \xi \\
&=& \int_{\RR} \widehat{p^*}(\xi) \overline{\widehat{p^*}(\xi) \e^{-\ri \xi b_{f, x, u}}} \left(\e^{\ri \xi (k_{f, x, u}+1)} + \e^{\ri \xi (k_{f, x, u} -1)}\right)d \xi \\
&=&\int_{\RR} \widehat{p^*}(\xi) \overline{\widehat{p^* (\cdot - b_{f, x, u})}(\xi)} \left(\e^{\ri \xi (k_{f, x, u}+1)} + \e^{\ri\xi (k_{f, x, u} -1)}\right)d \xi \\
&=& \int_{[-\pi, \pi)} \left\{\sum_{\ell \in \ZZ}\widehat{p^*}(\xi + 2 \ell \pi) \overline{\widehat{p^* (\cdot - b_{f, x,u})}(\xi + 2 \ell \pi)}\right\} \left(\e^{\ri \xi (k_{f, x, u}+1)} + \e^{\ri \xi
(k_{f, x, u} -1)}\right)d \xi \\
&=& \int_{[-\pi, \pi)} \left\{\sum_{\ell \in\ZZ} \langle p^*(\cdot -
\ell), p^* (\cdot - b_{f, x, u})\rangle_{L^2 (\RR)} e^{\ri \ell
\xi}\right\} \left(\e^{\ri \xi (k_{f, x, u}+1)} + \e^{\ri \xi (k_{f,
x, u} -1)}\right)d \xi,
\end{eqnarray*}
where we have used (\ref{identity}) in the last step. Since $b_{f,
x, u} \in [0, 1)$, we see easily that
$$\langle p^*(\cdot -
\ell), p^* (\cdot - b_{f, x, u})\rangle_{L^2 (\RR)}
=\begin{cases} 1-b_{f, x, u}, &
\hbox{if} \ \ell=0, \\
b_{f, x, u}, &\hbox{if} \ \ell=1, \\
0, &\hbox{if} \ \ell \in\ZZ\setminus\{0, 1\}. \end{cases}
$$ Hence
\begin{eqnarray*} &&\int_{\RR} \left|\widehat{p^*}(\xi)\right|^2 \left(\e^{\ri\xi} + \e^{-\ri\xi}\right)
\e^{\ri\xi (f(x) - f (u))} d \xi \\
&=& \int_{[-\pi, \pi)} \Big(1-b_{f, x, u} + b_{f, x, u} \e^{\ri \xi}\Big)  \left(\e^{\ri \xi (k_{f, x,
u}+1)} + \e^{\ri \xi (k_{f, x, u} -1)}\right) d \xi
\\
&=&\begin{cases} 2 \pi(1-b_{f, x, u}), & \quad
\hbox{if} \ k_{f, x, u}=1, -1, \\
2 \pi b_{f, x, u}, &\quad \hbox{if} \ k_{f, x, u}=0, -2, \\
0, &\quad \hbox{if} \ k_{f, x, u} \in\ZZ\setminus\{1, 0, -1, -2\}. \end{cases}
\end{eqnarray*}
Using the same procedure, we see that $\dint_{\RR} \left|\widehat{p^*}(\xi)\right|^2 \left(\e^{\ri\xi} + \e^{-\ri\xi}\right)\e^{-\ri\xi (f (x) - f (u))} d \xi$  has exactly the same value. Thus
\begin{eqnarray*}
&&-\frac{1}{4 \pi} \int_{\RR} \left|\widehat{p^*}(\xi)\right|^2 \left(\e^{\ri\xi} + \e^{-\ri\xi}\right) \left(\e^{\ri\xi (f (x) - f (u))} +
\e^{-\ri\xi (f (x) - f (u))}\right) d \xi \\
&=& \begin{cases} b_{f, x, u} -1, &
\quad\hbox{if} \ k_{f, x, u}=1, -1, \\
-b_{f, x, u}, &\quad \hbox{if} \ k_{f, x, u}=0, -2, \\
0, &\quad \hbox{if} \ k_{f, x, u} \in\ZZ\setminus\{1, 0, -1, -2\}. \end{cases}
\end{eqnarray*}
Denote \begin{eqnarray*} \Delta_1& =& \{(x, u)\in \cX_1\times\cX_2: 1 \leq f(x) - f(u) <2\}\bigcup \{(x, u)\in \cX_1\times\cX_2: -1 \leq f(x) - f(u) <0\},\\
 \Delta_2& = &\{(x, u)\in \cX_1\times\cX_2: 0 \leq f(x) - f(u) <1\}\bigcup \{(x, u)\in \cX_1\times\cX_2: -2 \leq f(x) - f(u) <-1\},\\
 \Delta_3 &= &\{(x, u)\in \cX_1\times\cX_2: f(x) - f(u) <-2\}\bigcup \{(x, u)\in \cX_1\times\cX_2: f(x) - f(u) \geq 2\}.
 \end{eqnarray*}
 Note that $k_{f, x, u}$ is the integer part of $f(x) - f(u)$. We have
\begin{eqnarray*} V_{12}(f) &=& \int_{{\mathcal X}_1} \int_{{\mathcal X}_2}
\Big\{\left(b_{f, x, u} -1\right) \id_{\Delta_1}(x,u)  - b_{f, x, u} \id_{\Delta_2}(x,u)\Big\} d \px (x) d \px (u).
\end{eqnarray*}
Since $0\leq b_{f, x, u} <1$, we see that $V_{12}(f)$ is minimized if and only if $b_{f, x, u} =0$, $\Delta_1=\cX_1\times \cX_2$ and $\Delta_2=\emptyset$. These conditions are equivalent to $f(x)-f(u)=k_{f,x,u}=\pm 1$  for almost all $(x, u) \in {\mathcal X}_1 \times {\mathcal X}_2$. Therefore, $V_{12}(f)$ is minimized if and only if $|f(x) - f(u)| =1$ for almost every $(x,
u) \in {\mathcal X}_1 \times {\mathcal X}_2$. In this case, the minimum value of $V_{12}(f)$ equals $$V_{12}^* :=-\px({\mathcal X}_1)
\px ({\mathcal X}_2) =-\frac{1}{4}.$$ Moreover, for any measurable function $f$, we have
\begin{eqnarray*} V_{12} (f) - V_{12}^* &=& \int_{{\mathcal X}_1} \int_{{\mathcal X}_2}
b_{f, x, u} \id_{\Delta_1}(x,u) + \left(1- b_{f, x, u}\right) \id_{\Delta_2}(x,u)
 + \id_{\Delta_3}(x,u) d
\px (x) d \px (u).
\end{eqnarray*}

On $\Delta_1$, we have $b_{f, x, u} = \left| |f(x) - f(u)|-1\right|$
and as a number on $[0, 1)$, it satisfies $b_{f, x, u} = \left|
|f(x) - f(u)|-1\right| \geq \left(|f(x) - f(u)|-1\right)^2$.
Similarly on $\Delta_2$ we have $1- b_{f, x, u} =\left| |f(x) -
f(u)|-1\right| \geq \left(|f(x) - f(u)|-1\right)^2$. On $\Delta_3$,
since the function $f$ takes values on $[-M, M]$, we have
$2\leq|f(x)-f(u)|\le 2M.$ Therefore $1 \geq
\frac{1}{4M^2}\left(|f(x) - f(u)|-1\right)^2$. Thus,
\begin{eqnarray*} V_{12} (f) - V_{12}^* & \geq & \frac{1}{4M^2} \int_{{\mathcal X}_1} \int_{{\mathcal
X}_2} \Big(|f(x) - f(u)|-1\Big)^2 d \px (x) d \px (u)\\
& \ge & \frac 1{48\pi^2M^3} \int_{{\mathcal X}_1} \int_{{\mathcal
X}_2} \Big(|f(x) - f(u)|-1\Big)^2 d \px (x) d \px (u),
\end{eqnarray*}
where we impose a lower bound in the last step in order to use \eqref{partIest} and \eqref{partIIest} later.

To bound $V_{12}(f)-V_{12}^*$ further, we need the following elementary inequality: for $A, a\in \RR$,
$$ A^2 =  a^2 + (A-a)^2 + 2 \frac{a}{\sqrt{2}} \sqrt{2}(A-a) \geq a^2 + (A-a)^2 - \frac{a^2}{2} - 2
(A-a)^2 =\frac{a^2}{2} - (A-a)^2. $$ Applying it with $A=|f(x) -
f(u)|-1$ and $a=\left|m_{f, {\mathcal X}_1} - m_{f, {\mathcal X}_2}\right| -1$
and using the fact
\begin{eqnarray*} \Big(|f(x) - f(u)|- \left|m_{f, {\mathcal X}_1} - m_{f,
{\mathcal X}_2}\right|\Big)^2 &\leq& \Big(\left(f(x) - m_{f,
{\mathcal X}_1}\right) -\left(f(u) - m_{f, {\mathcal
X}_2}\right)\Big)^2 \\
&\leq& 2\Big(f(x) - m_{f, {\mathcal X}_1}\Big)^2 + 2\Big(f(u)
- m_{f, {\mathcal X}_2}\Big)^2,
\end{eqnarray*}
we obtain
\begin{eqnarray*} \Big(|f(x)-f(u)|-1\Big)^2 & \ge &  \frac{1}{2} \Big(\left|m_{f, {\mathcal X}_1} - m_{f,
{\mathcal X}_2}\right| -1\Big)^2 - 2\Big(f(x) - m_{f, {\mathcal X}_1}\Big)^2 - 2\Big(f(u)
- m_{f, {\mathcal X}_2}\Big)^2.
\end{eqnarray*}
It follows to
\begin{equation}\label{partIIIest}
\begin{array}{rcl} V_{12} (f) - V_{12}^* &\geq& \dfrac{1}{48\pi^2M^3} \left\{\dfrac 1 8
\Big(\left|m_{f, {\mathcal X}_1} - m_{f,{\mathcal X}_2}\right| -1\Big)^2  -
\left\|f -m_{f, {\mathcal X}_1}\right\|^2_{\L2p({\mathcal X}_1)} \right. \\
& & \qquad\qquad \qquad  \left. -\left\|f - m_{f,{\mathcal X}_2}\right\|^2_{\L2p({\mathcal X}_2)} \right\}.
\end{array}
\end{equation}
Combining (\ref{partIest}),
(\ref{partIIest}), and \eqref{partIIIest}, we have with $c=\frac{1}{400\pi^2M^3} $,
\begin{equation}
 V(f) - V^* \geq c\left\{\Big(\! \left|m_{f, {\mathcal X}_1} - m_{f,
{\mathcal X}_2}\right| -1\Big)^2\!\! + \left\|f - m_{f, {\mathcal
X}_1}\right\|^2_{\L2p\!({\mathcal X}_1)}\!\! + \left\|f - m_{f,
{\mathcal X}_2}\right\|^2_{\L2p\!({\mathcal X}_2)}\right\}.
\label{totalest}
\end{equation}

With above preparations we can now prove our conclusions. Firstly, combining the conditions for minimizing
$V_{11}$, $V_{22}$ and $V_{12}$ we see easily the result in part (i).

 By $V^*=V^*_{11}+V^*_{22}+V^*_{12}=-\frac 58$ we get $\calR^*=-\log(\frac 58).$
For $f^*,$ a direct computation gives  $p_{\!_E} = \frac 14
\id_{[-\frac 32, -\frac 12]} +\frac 12 \id_{[-\frac 12, \frac 12]} +
\frac 14 \id_{[\frac 12, \frac 32]}$. So $\calR(f^*)=-\log(\frac
38)$ and we prove part (ii).

For any measurable function $f$, we take a function $f_\calR^*=f_1\id_{\cX_1} + f_2\id_{\cX_2}$
with $f_1=m_{f,\cX_1}$ and $f_2=f_1+f_{12},$ where $f_{12}$ is a constant defined to be 1 if $m_{f,\cX_2}\geq m_{f,\cX_1}$ and -1 otherwise.
Then $f_\calR^*\in\calF_\calR^*$ is a minimizer of the error entropy function $\calR(f).$
Moreover, it is easy to check that
$$\begin{array}{rl}
\|f-f_\calR^*\|_\L2p^2 & = \left\|f - m_{f, {\mathcal
X}_1}\right\|^2_{\L2p({\mathcal X}_1)} + \left\|f-f_2\right\|^2_{\L2p({\mathcal X}_2)}.
\end{array}$$
Since $\dint_{\cX_2}(f-m_{f, \cX_2})d\px=0$, we have
$$\|f-f_2\|_{\L2p(\cX_2)}^2=\dint_{\cX_2}(f-m_{f,\cX_2})^2d\px+\dint_{\cX_2}
(m_{f,\cX_2}-f_2)^2d\px.$$
Observe that
$m_{f, \cX_2}-f_2=m_{f, \cX_2}-m_{f, \cX_1}-f_{12}$ and by the choice of the constant $f_{12}$, we see that
$$|m_{f,\cX_2}-f_2|=\left||m_{f,\cX_2}-m_{f,\cX_1}|-1\right|.$$
Hence
$$\begin{array}{rl}
\|f-f_\calR^*\|_\L2p^2 & = \left\|f - m_{f, {\mathcal
X}_1}\right\|^2_{\L2p({\mathcal X}_1)} + \left\|f - m_{f, {\mathcal
X}_2}\right\|^2_{\L2p({\mathcal X}_2)} +\dfrac 12 \Big(\left|m_{f,
{\mathcal X}_1} - m_{f, {\mathcal X}_2}\right| -1\Big)^2.
\end{array}$$

This in combination with \eqref{totalest} leads to the conclusion in part (iii) with the constant $C'=400\pi^2M^3B_U$.

For part (iv), the first convergence is a direct consequence of the error entropy consistency.
To see the second one, it is suffices to notice
$$\min_{b\in\RR} \|f_\bz+b-f^*\|_\L2p = \min_{b\in\RR} \min_{f_\calR^*\in \calF_\calR^*}
\|f_\bz-f_\calR^*+f_\calR^*+b\|_\L2p \longrightarrow
\min_{b\in\RR} \min_{f_\calR^*\in \calF_\calR^*} \|f_\calR^*+b\|_\L2p,$$
which has the minimum value of $\frac 12$ achieved at $b=-\frac {f_1+f_2}{2}.$
\end{proof}

\section{Regression consistency}

In this section we prove that the regression consistency is true for
both homoskedastic models and heteroskedastic models when the
bandwidth parameter $h$ is chosen to tend to infinity in a suitable
rate. We need the following result proved in \cite{HFWZ2012}.

\begin{proposition}\label{prop:Ehlower}
There exists a constant $C''$ depending only on $\calH, \rho$ and M such that
$$\|f+\bE(f^*-f)-f^*\|_\L2p^2 \le C''\left(h^3\left(\calE_{h}(f)-\calE_h^*\right)+\frac 1{h^2}\right), \qquad \forall f\in\calH,\quad h>0,$$
where $\calE_h^*=\min\limits_{f\in\calH} \calE_h(f).$
%and
%$$\min\limits_{f\in\H} \calE_h(f) - \calE_h(f^*) \le \min\limits_{f\in\H}\left\|f-f^*-\bE(f-f^*)\right\|_\L2p
%+\frac {2D}{h^2}.$$
\end{proposition}

Theorem \ref{thm:regressionconsistency} is an easy consequence of Propositions \ref{prop:Ehlower}
and \ref{prop:szbound}. To see this, it suffices to notice that
$\calE_{h}(f_\bz)-\calE_h^*\le 2\calS_\bz.$

%let $f_{h}$ be the minimizer of $\calE_h(f)$ over $\calH.$ Then we have
%$$\begin{array}{rl} \calE_h(f_\bz) -\calE_h(f^*) & = \Big(\calE_h(f_\bz)-\calE_{h,\bz}(f_\bz)\Big)
%+\Big(\calE_{h,\bz}(f_\bz)-\calE_{h,\bz}(f_h) \Big) \\
%& \qquad + \Big(\calE_{h,\bz}(f_h)-\calE_h(f_h)\Big)+ \Big( \calE_h(f_h)-\calE_h(f^*)\Big)  \\
%& \le 2\calS_\bz + \dfrac{2D}{h^2}.
%\end{array}$$

\section{Regression consistency for two special models}

In previous sections we see the information error $\calE_h(f)$ plays a very important role in analyzing the empirical MEE algorithm. Actually, it is of independent interest as a loss function to the regression problem. As we discussed, as $h$ tends to 0, $\calE_h(f)$ tends to $V(f)$ which is the loss function used in the MEE algorithm. As $h$ tends to $\infty$, it behaves like a least square ranking loss \cite{HFWZ2012}. 
In this section we use it to study the regression consistency of MEE for the two classes of special models $\calP_1$ and $\calP_2$.

\subsection{Symmetric unimodal noise model}

In this subsection we prove the regression consistency for
the symmetric unimodal noise case stated in Theorem \ref{thm:symuninoise}. To this end, We need the following two lemmas of which the first is from \cite{Laha:1961}. Let $f*g$ denotes the convolution of two integrable functions $f$ and $g$.

\begin{lemma}
\label{lema:diffpdf}
The convolution of two symmetric unimodal distribution functions is symmetric unimodal.
\end{lemma}

\begin{lemma}
\label{lem:denfun} Let $\epsilon_x=y-f^*(x)$ be the noise random
variable at x and denote $g_{x,u}$ as the probability density
function of $\epsilon_x-\epsilon_u$ for $x,u\in\cX$ and
$\widehat{g_{x,u}}$ as the Fourier transform of $g_{x,u}$. If
$\rho$ belongs to $\calP_1$, we have
\begin{enumerate}
\item $g_{x,u}$ is symmetric and unimodal for $x,u\in\cX$;
\item $\widehat{g_{x,u}}(\xi)$ is nonnegative for $\xi\in\RR$;
\item $\widehat{g_{x,u}}(\xi)\geq C_0$ for $\xi\in[-c_0,c_0]$, where $c_0, C_0$ are two positive constants.
\end{enumerate}

\end{lemma}
\begin{proof}
Since both $p_{\epsilon|X}(\cdot|x)$ and $p_{\epsilon|X}(\cdot|u)$ are symmetric and unimodal, (i) is an easy consequence of Lemma \ref{lema:diffpdf}.
With the symmetry property, $-\epsilon_u$ has the same density function as $\epsilon_u$, so we have $g_{x,u}=p_{\epsilon|X}(\cdot|x)*p_{\epsilon|X}(\cdot|u)$, which implies
\begin{eqnarray*}
\widehat{g_{x,u}}(\xi)=\widehat{p_{\epsilon|X=x}}(\xi)\widehat{p_{\epsilon|X=u}}(\xi).
\end{eqnarray*}
Since $\rho$ is in $\calP_1$, we easily see that $\widehat{g_{x,u}}(\xi)$ is nonnegative for $\xi\in\RR$ and that for some positive constants $c_0, C_0$, there holds $\widehat{g_{x,u}}(\xi)\geq C_0$ for $\xi\in[-c_0,c_0]$.
\end{proof}

The following result gives some regression consistency analysis for the MEE algorithm where the bandwidth parameter $h$ is fixed. It immediately implies Theorem \ref{thm:symuninoise} stated in the second section.

\begin{proposition}
\label{prop:Symuninoise}
Assume $\rho$ belongs to $\calP_1$. Then for any \emph{fixed} $h$
\begin{enumerate}
\item $f^*+b$ is a minimizer of $\calE_h(f)$ for any constant b;
\item there exists a constant $C_h>0$ such that
\begin{eqnarray}
\label{CTfixhsun}
%{\bf{var}} [f-f^*]
\|f+\bE(f^*-f)-f^*\|_\L2p^2 \leq
C_h(\mathcal{E}_h(f)-\mathcal{E}_h(f^*)),\qquad \forall f\in \mathcal{H};
\end{eqnarray}
%where ${\bf{var}}$ denotes the variance;
\item with probability at least 1-$\delta$, there holds
\begin{eqnarray}
\label{convergesun}
%{\bf{var}} [f_{\bf{z}}-f^*]
\|f_\bz+\bE_x(f^*-f_\bz)-f^*\|_\L2p^2 \leq\frac{2BC_h}{h^2\sqrt{n}}+\frac{\sqrt{2}C_h}{h\sqrt{n}}
\sqrt{\log(1/\delta)},
\end{eqnarray}
where B is given explicitly in the appendix.
\end{enumerate}
\end{proposition}

\begin{proof}
Recall that $\epsilon_x=y-f^*(x)$, $\epsilon_u=v-f^*(u)$ and $g_{x,u}$ is the probability density function of $\epsilon_x-\epsilon_u$. We have for any measurable function f,
\begin{eqnarray*}
\calE_h(f)&=&-
\dint_\cZ\dint_\cZ G_h\Big((y-f(x))-(v-f(u))\Big) d\rho(x,y)d\rho(u,v)\\
&=&\frac{1}{\sqrt{2\pi}h}\int_\cX\int_\cX
\left[-\int_{-\infty}^{\infty}
\exp\left(-\frac{(w-t)^2}{2h^2}\right)g_{x,u}(w)dw\right] d\px(x)d\px(u)
\end{eqnarray*}
where $t=f(x)-f^*(x)-f(u)+f^*(u)$.

Now we apply the Planchel formula and find
\begin{eqnarray*}
&&\calE_h(f)-\calE_h(f^*)\\&=&\frac{1}{\sqrt{2\pi}h}\int_\cX\int_\cX\left[\int_\RR
\exp\left(-\frac{w^2}{2h^2}\right)g_{x,u}(w)d w-\int_\RR \exp\left(-\frac{w^2}{2h^2}\right)g_{x,u}(w+t)dw\right]
d\px(x)d\px(u)\\
&=&\frac{1}{2\pi}\int_\cX\int_\cX\int_\RR \exp\left(-\frac{h^2\xi^2}{2}\right)\widehat{g_{x,u}}(\xi)\left(1-\e^{\ri\xi(f(x)-f^*(x)-f(u)+f^*(u))}\right)d\xi
d\px(x)d\px(u)\\
&=&\frac{1}{2\pi}\int_\cX\int_\cX\int_\RR \exp\left(-\frac{h^2\xi^2}{2}\right)\widehat{g_{x,u}}(\xi)2\sin^2\frac{\xi\left(f(x)-f^*(x)-f(u)+f^*(u)\right)}
{2}d\xi d\px(x)d\px(u).
\end{eqnarray*}
By Lemma \ref{lem:denfun}, $\widehat{g_{x,u}}(\xi)\geq0$ for $\xi\in\RR$. So $\calE_h(f)-\calE_h(f^*)\geq0$ for any measurable function $f.$ This tells us that $f^*$ and $f^*+b$ for any $b\in\RR$ are minimizers of $\calE_h(f)$.

To prove \eqref{CTfixhsun} we notice that both $f$ and $f^*$ take values on $[-M, M].$ Hence $|f(x)-f^*(x)-f(u)+f^*(u)|\leq4M$ for any $x, u\in\cX$. So when $|\xi|\leq\frac{\pi}{4M}$, we have
\begin{eqnarray*}
\left|\frac{\xi(f(x)-f^*(x)-f(u)+f^*(u))}{2}\right|\leq\frac\pi2,
\end{eqnarray*}
and
\begin{eqnarray*}
\left|\sin\frac{\xi(f(x)-f^*(x)-f(u)+f^*(u))}{2}\right|\geq\frac2\pi\left|\frac{\xi
(f(x)-f^*(x)-f(u)+f^*(u))}{2}\right|.
\end{eqnarray*}
Then we have
\begin{eqnarray*}
&& \int_\RR \exp\left(-\frac{h^2\xi^2}{2}\right)\widehat{g_{x,u}}(\xi)2\sin^2\frac{\xi\left(f(x)-f^*(x)-f(u)+f^*(u)\right)}
{2}d\xi\\
&&\geq\int_{|\xi|\leq\frac{\pi}{4M}}\exp\left(-\frac{h^2\xi^2}{2}\right)\widehat{g_{x,u}}(\xi)2\sin^2\frac{\xi\left
(f(x)-f^*(x)-f(u)+f^*(u)\right)}{2}d\xi\\
&&\geq\int_{|\xi|\leq\frac{\pi}{4M}}\exp\left(-\frac{h^2\xi^2}{2}\right)\widehat{g_{x,u}}(\xi)\frac2{\pi^2}\xi^2(f(x)-
f^*(x)-f(u)+f^*(u))^2d\xi\\
&&\geq\int_{|\xi|\leq \min\{\frac{\pi}{4M}, c_0\}}\exp\left(-\frac{h^2\xi^2}{2}\right)\widehat{g_{x,u}}(\xi)\frac2{\pi^2}\xi^2(f(x)-
f^*(x)-f(u)+f^*(u))^2d\xi\\
&&\geq\frac{2C_0}{\pi^2}\int_{|\xi|\leq \min\{\frac{\pi}{4M}, c_0\}}\exp\left(-\frac{h^2\xi^2}{2}\right)\xi^2(f(x)-
f^*(x)-f(u)+f^*(u))^2d\xi.
\end{eqnarray*}
Therefore, using \eqref{varf}
\begin{eqnarray*}
\calE_h(f)-\calE_h(f^*)\geq\left(\frac{2C_0}{\pi^3}\int_{|\xi|\leq \min\{\frac{\pi}{4M}, c_0\}}\xi^2
\exp\left(-\frac{h^2\xi^2}{2}\right)d\xi\right){\bf{var}} \|f+\bE(f^*-f)-f^*\|_\L2p^2.
\end{eqnarray*}
%where the inequality follows from (see \cite{HFWZ2012})
%\begin{eqnarray}
%\label{var}
%2{\bf{var}} [f-f^*]=\int_\cX\int_\cX(f(x)-
%f^*(x)-f(u)+f^*(u))^2d\px(x)d\px(u).
%\end{eqnarray}
Since
$c_h=\int_{|\xi|\leq \min\{\frac{\pi}{4M}, c_0\}}\xi^2 \exp\left(-\frac{h^2\xi^2}{2}\right)d\xi$
is positive, \eqref{CTfixhsun} follows by taking $C_h=\frac{\pi^3}{2c_hC_0}$. 

With (\ref{CTfixhsun}) valid, (iii) is an easy consequence of Proposition \ref{prop:szbound}. 
\end{proof}

\subsection{Symmetric bounded noise models}

In this subsection we prove the regression consistency for
the symmetric bounded noise models stated in Theorem \ref{thm:boundnoise}.
\begin{proposition}
We assume $\rho$ belongs to $\calP_2$. Then there exists a constant $h_{\rho,\calH}>0$ such that for any \emph{fixed} $h>h_{\rho,\calH}$ the following holds:
\begin{enumerate}
\item $f^*+b$ is the minimizer of $\calE_h(f)$ for any constant b;
\item there exists a constant $C_2>0$ depending only on $\rho, \calH, \widetilde{M}$, M and h such that
\begin{eqnarray}
\label{CTfixhbn}
%{\bf{var}} [f-f^*]\leq
\|f+\bE(f^*-f)-f^*\|_\L2p^2 \le 
C_2(\mathcal{E}_h(f)-\mathcal{E}_h(f^*)),\qquad\forall f\in \mathcal{H};
\end{eqnarray}
\item with probability at least 1-$\delta$, there holds
\begin{eqnarray}
\label{convergebn}
%{\bf{var}} [f_{\bf{z}}-f^*]\leq
\|f_\bz+\bE_x(f^*-f_\bz)-f^*\|_\L2p^2 \le
\frac{2BC_2}{h^2\sqrt{n}}+\frac{\sqrt{2}C_2}{h\sqrt{n}}
\sqrt{\log(1/\delta)}.
\end{eqnarray}
\end{enumerate}
\end{proposition}

\begin{proof}
Since $\rho$ belongs to $\calP_2$, we know that $\epsilon_x$ is supported on $[-\widetilde{M}, \widetilde{M}]$ and $g_{x,u}$ on $[-2\widetilde{M}, 2\widetilde{M}]$. So for any measurable function $f:\cX\rightarrow\RR$,
$$\calE_h(f)=\frac{1}{\sqrt{2\pi}h}\int_\cX\int_\cX
T_{x,u}(f(x)-f^*(x)-f(u)+f^*(u)) d\px(x)d\px(u),$$
where $T_{x,u}$ is a univariate function given by
$$T_{x,u}(t)=-\int_{-2\widetilde{M}}^{2\widetilde{M}} \exp\left(-\frac{(w-t)^2}{2h^2}\right)g_{x,u}(w)dw.$$
Observe that
\begin{eqnarray*}
T'_{x,u}(t) &=& -\dint_{-2\widetilde{M}}^{2\widetilde{M}}\exp\left(-\frac{(w-t)^2}{2h^2}\right)\left(\frac{w-t}{h^2}\right)g_{x,u}(w)dw\\
&=&-\frac1{h^2}\dint_{0}^{2\widetilde{M}}w\exp\left(-\frac{w^2}{2h^2}\right)[g_{x,u}(w+t)-g_{x,u}(w-t)]d w,
\end{eqnarray*}
and
\begin{eqnarray*}
T''_{x,u}(t)
=-\frac1{h^2}\dint_{-2\widetilde{M}}^{2\widetilde{M}}\exp\left(-\frac{(w-t)^2}{2h^2}\right)\left[\frac{(w-t)^2}{h^2}
-1\right]g_{x,u}(w)d w.
\end{eqnarray*}
So $T'_{x,u}(0)=0$. Moreover, if we choose $h_{\rho, \calH}:=4M+2\widetilde{M}$, then for $h>h_{\rho,\calH}$ and $|t|\le 4M$, 
\begin{eqnarray*}
T_{x,u}^{''}(t)&\geq&\frac{1}{h^2} \left(1-\frac{(4M+2\widetilde{M})^2}{h^2}\right) \exp\left(-\frac{2(2M+\widetilde{M})^2}{h^2}\right)
\dint_{-2\widetilde{M}}^{2\widetilde{M}}g_{x,u}(w)dw\\
&=& \frac{1}{h^2} \left(1-\frac{(4M+2\widetilde{M})^2}{h^2}\right) \exp\left(-\frac{2(2M+\widetilde{M})^2}{h^2}\right) >0.
\end{eqnarray*}
So $T_{x,u}$ is convex on $[-4M, 4M]$ and $t=0$ is its unique minimizer.
By the fact $t=f(x)-f^*(x)-f(u)+f^*(u)\in[-4M, 4M]$ for all $x,u\in\cX,$ 
we conclude that, for any constant $b$, $f^*+b$ is the minimizer of $\calE_h(f)$.

%If we choose $h'_{\rho, \calH}=2\widetilde{M}\geq|w|$, then $T''_{x,u}(0)$ is positive for $h>h'_{\rho, \calH}$, which

By Taylor expansion, we obtain
\begin{eqnarray*}
T_{x,u}(t) - T_{x,u}(0) =
T_{x,u}^{'}(0)t+\frac{T_{x,u}^{''}(\xi)}{2}t^2=\frac{T_{x,u}^{''}(\xi)}{2}t^2, \qquad t\in[-4M, 4M],
\end{eqnarray*}
where $\xi$ is between 0 and $t.$ So $|\xi|\leq|t|\leq4M$. 
%Observe that for $h>h''_{\rho, \calH}:=4M+2\widetilde{M}$,
%\begin{eqnarray*}
%T_{x,u}^{''}(\xi)&\geq&\frac{1}{h^2}(1-\frac{(4M+2\widetilde{M})^2}{h^2})\exp\left(-\frac{2(2M+\widetilde{M})^2}{h^2}\right)
%\dint_{-2\widetilde{M}}^{2\widetilde{M}}g_{x,u}(w)dw\\
%&=& \frac{1}{h^2}(1-\frac{(4M+2\widetilde{M})^2}{h^2})\exp\left(-\frac{2(2M+\widetilde{M})^2}{h^2}\right).
%\end{eqnarray*}
It follows that with the constant
$C_2=h^2 \exp\left(\frac{2(2M+\widetilde{M})^2}{h^2}\right)/(1-\frac{(4M+2\widetilde{M})^2}{h^2})$
independent of $x$ and $u$ we have
\begin{eqnarray*}
t^2\leq
2C_2\left[T_{x,u}(t)-T_{x,u}(0)\right].
\end{eqnarray*}
By virtue of the equality \eqref{varf},
\begin{eqnarray*}
%{\bf{var}} [f-f^*]\leq
\|f+\bE(f^*-f)-f^*\|_\L2p^2 \le 
C_2(\mathcal{E}_h(f)-\mathcal{E}_h(f^*)).
\end{eqnarray*}
Together with Proposition \ref{prop:szbound}, \eqref{CTfixhbn} leads to \eqref{convergebn}. Theorem \ref{thm:boundnoise} has been proved by taking $h_{\rho, \calH}=4M+2\widetilde{M}$.
\end{proof}

\appendix
\section*{Appendix: Proof of Proposition \ref{prop:szbound}
\addtocounter{section}{1}
\setcounter{theorem}{0}
\setcounter{equation}{0}
}

In this appendix we prove Proposition \ref{prop:szbound}.
Let us first give the definition of the empirical covering number
which is used to characterize the capacity of the hypothesis space
and prove the sample error bound.

The $\ell_2$-norm empirical covering number is defined by means of
the normalized $\ell_2$-metric $d_2$ on the Euclidian space $\RR^n$
given by
$$d_2({\bf a}, {\bf b})=\left(\frac 1n \sum_{i=1}^n
|a_i-b_i|^2\right)^{1/2}$$ for ${\bf a}=(a_i)_{i=1}^n, {\bf
b}=(b_i)_{i=1}^n \in \RR^n.$

\begin{definition} For a subset S of a pseudo-metric space $(\mathcal
M,d)$ and $\varepsilon>0$, the covering number $\calN(S,
\varepsilon,d)$ is defined to be the minimal number of balls of
radius $\varepsilon$ whose union covers S. For a set $\calH$ of
bounded functions on $\cX$ and $\varepsilon>0$, the $\ell_2$-norm
empirical covering number of $\calH$ is given by
\begin{eqnarray}\label{coveringnumber}
\calN_2(\calH,\varepsilon)=\sup_{n\in\NN}\sup_{\bf
x\in\cX^n}\calN(\calH|_{\bf x},\varepsilon,d_2).
\end{eqnarray}
where for $n\in\NN$ and ${\bf x} =(x_i)_{i=1}^n\in \cX^n,$ we denote
the covering number of the subset $\calH|_{\bf
x}=\{(f(x_i))_{i=1}^n:\; f\in\calH\}$ of the metric space $(\RR^n,
d_2)$ as $\calN(\calH|_{\bf x}, \varepsilon, d_2)$.
\end{definition}

\begin{definition}
Let $\rho$ be a probability measure on a set $\mathcal X$ and
suppose that $X_1, ..., X_n$ are independent samples selected
according to $\rho$. Let $\H$ be a class of functions mapping from
$\mathcal X$ to $\mathbb{R}$. Define the random variable
\begin{eqnarray}\label{Rademacher}
\hat{\mathcal
R}_n(\H)=\bE_\sigma\left[\sup_{f\in\H}|\frac1n\sum\limits_{i=1}^{n}\sigma_if(X_i)|\Big|
X_1, ..., X_n\right],
\end{eqnarray}
where $\sigma_1, ..., \sigma_n$ are independent uniform
$\{\pm1\}$-valued random variables. Then the Rademacher average
\cite{Bart:Mend:2002} of $\H$ is $\mathcal{R}_n(\H)=\bE\hat{\mathcal
R}_n(\H)$.
\end{definition}

The following lemma from \cite{Bart:Bous:Mend:2005} shows that these two
complexity measures we just defined are closely related.

\begin{lemma}
For a bounded function class $\H$ on $\mathcal{X}$ with bound M, and
$\calN_2(\H,\varepsilon)$ is $\ell_2$-norm
empirical covering number of $\calH$, then there exists a constant $C_1$ such that for
every positive integer n the following holds:
\begin{eqnarray}\label{entropyintegral}
\hat{\mathcal R}_n(\H)\leq
C_1\int_{0}^{M}\left(\frac{\log\calN_2(\H,\varepsilon)}n\right)^{1/2}d\varepsilon.
\end{eqnarray}
\end{lemma}
Moreover, we need the following lemma for Rademacher average.
\begin{lemma}\label{Lipschitz}
(1) For any uniformly bounded function f,
\begin{eqnarray*}
{\mathcal R}_n(\H+f)\leq {\mathcal R}_n(\H)+\|f\|_\infty/\sqrt{n}.
\end{eqnarray*}
(2) Let $\{\phi_i\}_{i=1}^{n}$ be functions with Lipschitz constants
$\gamma_i$, then \cite{Meir:Zhan:2003} gives
\begin{eqnarray*}
\bE_\sigma\{\sup_{f\in\mathcal{H}}\sum\limits_{i=1}^{n}\sigma_i\phi_i(f(x_i))\}\leq
\bE_\sigma\{\sup_{f\in\mathcal{H}}\sum\limits_{i=1}^{n}\sigma_i\gamma_if(x_i)\}.
\end{eqnarray*}
\end{lemma}

By applying McDiarmid's inequality we have the following proposition.

\begin{proposition}\label{McDiarmid}
For every $\varepsilon_1>0$, we have
$$\mathbb P\{\calS_\bz-\bE\calS_\bz>\varepsilon_1\}\leq \exp(-2nh^2\varepsilon_1^2).$$
\end{proposition}

\begin{proof}
Recall $$\calS_\bz=\sup_{f\in\H}|\calE_{h,\bz}(f)-\calE_{h}(f)|.$$
Let $i\in\{1,\cdots,n\}$ and
$\tilde\bz=\{z_1,\cdots,z_{i-1},\tilde{z}_{i},z_{i+1},\cdots,z_n\}$
be identical to $\bz$ except the i-th sample. Then
\begin{eqnarray*}
|\calS_\bz-\calS_{\tilde{\bz}}|&\leq& \sup_{(x_i,
y_i)_{i=1}^{n},(\tilde{x}_i, \tilde{y}_i)}
\left|\sup_{f\in\H}|\calE_{h,\bz}(f)-\calE_{h}(f)|-\sup_{f\in\H}
|\calE_{h,\tilde{\bz}}(f)-\calE_{h}(f)|\right|\\
&\leq& \sup_{(x_i, y_i)_{i=1}^{n},(\tilde{x}_i,
\tilde{y}_i)}\sup_{f\in\H}
|\calE_{h,\bz}(f)-\calE_{h,\tilde{\bz}}(f)|\\
&\leq& \frac{1}{n^2}\sum_{j=1}^{n}\sup_{(x_i,
y_i)_{i=1}^{n},(\tilde{x}_i, \tilde{y}_i)}\sup_{f\in\H}
|G_h(e_i,e_j)-
G_h(\tilde{e}_i, e_j)|\\
&\leq& \frac{1}{nh}.
\end{eqnarray*}
Then the proposition follows immediately from McDiarmid's
inequality.
\end{proof}

Now we need to bound $\bE\calS_\bz$.

\begin{proposition}\label{meanszbound}
$$\bE\calS_\bz\leq\frac{2}{\sqrt{\pi}h^2}\left(\frac{M}{\sqrt{n}}+
\mathcal{R}_n(\H)\right)+\frac{2}{\sqrt{2\pi}h n}.$$
\end{proposition}

\begin{proof}
 Let
$\eta(x,y,u,v)=\frac{1}{\sqrt{2\pi}}\exp(-\frac{[(y-f(x))-(v-f(u))]^2}{2h^2})$
for simplicity. Then
$$\calE_{h,\bz}(f)=-\frac{1}{n^2h}\sum_{i=1}^{n}\sum_{j=1}^{n}\eta(x_i,
y_i, x_j, y_j)$$ and
$$\calE_h(f)=-\frac{1}{h}\bE_{(x,y)}\bE_{(u,v)}\eta(x,y,u,v).$$
Then
\begin{eqnarray*}
h\calS_{\bz} &=&
h\sup_{f\in\H}\left|\calE_{h,\bz}(f)-\calE_h(f)\right|\\
&\leq&\sup_{f\in\H}\left|\bE_{(x,y)}\bE_{(u,v)}\eta(x,y,u,v)-\frac{1}{n}\sum_{j=1}^{n}\bE_{(x,y)}\eta(x,y,x_j,y_j)\right|\\
&+&
\sup_{f\in\H}\left|\frac{1}{n}\sum_{j=1}^{n}\bE_{(x,y)}\eta(x,y,x_j,y_j)-\frac{1}{n^2}\sum_{i=1}^{n}\sum_{j=1}^{n}\eta(x_i,
y_i, x_j, y_j)\right|\\
&\leq&\bE_{(x,y)}\sup_{f\in\H}\left|\bE_{(u,v)}\eta(x,y,u,v)-\frac{1}{n}\sum_{j=1}^{n}\eta(x,y,x_j,y_j)\right|\\
&+&\frac{1}{n}\sum_{j=1}^{n}\sup_{(u,v)\in\bz}\sup_{f\in\H}\left|\bE_{(x,y)}\eta(x,y,u,v)-\frac{1}{n-1}\sum_{i=1\atop
i\neq j}^{n}\eta(x_i,y_i,u,v)\right|\\
&+&\frac{1}{n}\sum_{j=1}^{n}\sup_{f\in\H}\left[\frac{1}{n}\eta(x_j,y_j,x_j,y_j)+\frac{1}{n(n-1)}\sum_{i=1\atop
i\neq j}^{n}\eta(x_i,y_i,x_j,y_j)\right]\\
&:=& S_1+S_2+S_3.
\end{eqnarray*}
Noting that
$$|\exp(-(y_i-f(x_i))^2)-\exp(-(y_i-g(x_i))^2)|\leq|f(x_i)-g(x_i)|,$$  we
have
\begin{eqnarray*}
\bE S_1 &=&
\bE_{(x,y)}\EE\sup_{f\in\H}\left|\EE_{(u,v)}\eta(x,y,u,v)-\frac{1}{n}\sum_{j=1}^{n}\eta(x,y,x_j,y_j)\right|\\
&\leq& \frac{2}{\sqrt{2\pi}}\sup_{(x,y)\in
\bz}\bE\bE_\sigma\sup_{f\in\H}\left|\frac{1}{n}\sum_{j=1}^{n}\sigma_j\exp(-\frac{[(y-f(x))-(y_j-f(x_j))]^2}{2h^2})\right|\\
&\leq& \frac{1}{h\sqrt{\pi}}\sup_{x\in
\mathcal X}\bE\bE_\sigma\sup_{f\in\H}\left|\frac{1}{n}\sum_{j=1}^{n}\sigma_j(f(x)-f(x_j))\right|\\
&\leq& \frac{1}{h\sqrt{\pi}}\left[\sup_{x\in
X}\bE_\sigma\sup_{f\in\H}\left|\frac{1}{n}\sum_{j=1}^{n}\sigma_jf(x)\right|+\bE\bE_\sigma\sup_{f\in
\H}\left|\frac{1}{n}\sum_{j=1}^{n}\sigma_jf(x_j)\right|\right]\\
&\leq&
\frac{1}{h\sqrt{\pi}}\left(\frac{M}{\sqrt{n}}+\mathcal{R}_n(\H)\right),
\end{eqnarray*}
where the second inequality is from Lemma \ref{Lipschitz}.
Similarly,
\begin{eqnarray*}
\bE S_2 &=&
\frac{1}{n}\sum_{j=1}^{n}\sup_{(u,v)\in\bz}\bE\sup_{f\in\H}\left|\bE_{(x,y)}\eta(x,y,u,v)-\frac{1}{n-1}\sum_{i=1\atop
i\neq j}^{n}\eta(x_i,y_i,u,v)\right|\\
&\leq& \frac{2}{n\sqrt{2\pi}}\sum_{j=1}^{n}\sup_{(u,v)\in
\bz}\bE\bE_\sigma\sup_{f\in\H}\left|\frac{1}{n-1}\sum_{i=1\atop
i\neq j}^{n}\sigma_i\exp(-\frac{[(y_i-f(x_i))-(v-f(u))]^2}{2h^2})\right|\\
&\leq& \frac{1}{nh\sqrt{\pi}}\sum_{j=1}^{n}\sup_{u\in
\mathcal{X}}\bE\bE_\sigma\sup_{f\in\H}\left|\frac{1}{n-1}\sum_{i=1\atop
i\neq j}^{n}\sigma_i(f(x_i)-f(u))\right|\\
&\leq& \frac{1}{nh\sqrt{\pi}}\sum_{j=1}^{n}\left[\sup_{u\in
X}\bE_\sigma\sup_{f\in\H}\left|\frac{1}{n-1}\sum_{i=1\atop i\neq
j}^{n}\sigma_i
f(u)\right|+\bE\bE_\sigma\sup_{f\in\H}\left|\frac{1}{n-1}\sum_{i=1\atop
i\neq j}^{n}\sigma_if(x_i)\right|\right]\\
&=&
\frac{1}{h\sqrt{\pi}}\left(\frac{M}{\sqrt{n}}+\mathcal{R}_n(\H)\right).
\end{eqnarray*}
It's easy to obtain $\bE S_3\leq\frac{2}{n\sqrt{2\pi}}$. Combining the
estimates for $S_1$, $S_2$, $S_3$ completes the proof.
\end{proof}

Now we can prove  Proposition \ref{prop:szbound}.

If $\H$ is MEE admissible, (\ref{entropyintegral}) leads to
\begin{eqnarray*}
\mathcal{R}_n(\H) &=& \bE\hat{\mathcal R}_n(\H) \leq
\frac{C_1}{\sqrt{n}}\int_{0}^{M}\bE\sqrt{\log\calN_2(\H,
\varepsilon)}d\varepsilon\\
&\leq& \frac{C_1}{\sqrt{n}}\int_{0}^{M}\sqrt{\bE\log\calN_2(\H,
\varepsilon)}d\varepsilon\\
&\leq& \frac{C_1\sqrt{c}}{\sqrt{n}}\int_{0}^{M}\varepsilon^{-s/2} d\varepsilon\\
&=& \left(\frac{2C_1\sqrt{c}}{2-s}M^{1-s/2}\right)\frac{1}{\sqrt{n}}.
\end{eqnarray*}
Let
$B=\frac{4C_1\sqrt{c}}{(2-s)\sqrt{\pi}}M^{1-s/2}+\frac{2M+\sqrt2}{\sqrt{\pi}}$,
combining Proposition \ref{McDiarmid} and Proposition
\ref{meanszbound} yields the desired result.

\bibliographystyle{abbrv}
\bibliography{MEE}

\end{document}